%% file: arxiv_new.tex
\documentclass{article} % For LaTeX2e
\usepackage{arxiv,times}

\input{ICLR2027/math_commands}

\usepackage{hyperref}
\usepackage{url}
\usepackage{amsmath}
\usepackage{booktabs}       % professional-quality tables

\newtheorem{proposition}{Proposition}

\newtheorem{remark}{Remark}
\newtheorem{proof}{Proof}
\usepackage{hyperref}
\usepackage{graphicx}
\usepackage{multirow}
\usepackage[table]{xcolor}
\usepackage{subcaption}

\title{Reshape and Recur: Improving SSMs with\\ Input Reshaping and Depth Recurrence}
\author{%
  Mónika Farsang\textsuperscript{\rm 1,2} \\
  \And
  Ramin Hasani\textsuperscript{\rm 2,3} \\
  \And
  Daniela Rus\textsuperscript{\rm 2} \\
  \And
  Radu Grosu\textsuperscript{\rm 1} \\
}

\begin{document}

\maketitle
\let\thefootnote\relax\footnotetext{\textsuperscript{\rm 1}TU Wien}
\let\thefootnote\relax\footnotetext{\textsuperscript{\rm 2}MIT CSAIL}
\let\thefootnote\relax\footnotetext{\textsuperscript{\rm 3}Liquid AI}
\let\thefootnote\relax\footnotetext{\textsuperscript{\rm }Corr. author: monika.farsang@tuwien.ac.at}

\begin{abstract}
State Space Models (SSMs) are increasingly deployed in the Edge because they offer, at comparable performance, a smaller memory/training/inference footprint, compared to Large Language Models (LLMs). These three advantages are a direct consequence of the time recurrence inherent in the SSMs architecture. Here, we further improve this recurrent architecture by positively answering two previously underexplored, orthogonal questions: (1)~Can we reduce SSMs memory-footprint without any performance penalty, by also employing depth recurrence? (2)~Can we increase SSMs performance by using a fixed and consistent time-granularity across all tasks?  The first question is somewhat unexpected, given that SSMs are already recurrent. However, the orthogonal depth recurrence further decreases SSMs memory footprint. We show that a looped SSM with $k$ parameters adaptively iterated $M$ times, achieves a performance comparable to a standard SSM with $k\,{\cdot}\,L$ independent parameters, where $M\,{\leq}\,L$.
The second question is also unexpected given the time-recurrent nature of the SSMs architecture. However, it makes perfect sense for the time-parallel training of SSMs on the entire input sequence. We show that concatenating time steps for lower-dimensional sequence elements, or flattening and re-chunking the joint feature-time dimension for high-dimensional ones, can improve the baseline by enhancing the way information is presented to the model.  Our results for both extensions lead to consistent benefits across four representative SSM architectures: LRU, S5, LinOSS, LrcSSM. 
\end{abstract}

\section{Introduction}
State Space Models (SSMs) have emerged as a compelling alternative to (transformer-based) Large Language Models (LLMs), for tasks requiring efficient processing of long sequences \citep{gu2022efficientlymodelinglongsequences, hasani2022liquid, orvieto2023resurrecting}. By maintaining a compact recurrent state updated across the sequence dimension, SSMs have a very small memory footprint, linear-time-learning and constant-time-inference complexity, respectively, while retaining strong empirical performance. These desirable properties of the SSMs have been demonstrated across a wide range of domains, including language, audio, and time series, and have lead to their increased use in the Edge~\citep{gu2023mamba, rusch2024oscillatory, farsang2025parallelizationnonlinearstatespacemodels,liquidAImercedes}.

In this paper, we demonstrate how to further improve SSMs accuracy and reduce SSMs memory footprint, by positively answering two previously underexplored and orthogonal questions, regarding the recurrent architecture of SSMs: (1)~Can we still reduce SSMs memory-footprint without any performance penalty by also taking advantage of layers-depth recurrence? (2)~Can we still increase SSMs performance by using a fixed and consistent time-granularity across all tasks?

Asking the first question is not trivial, given that SSMs are already recurrent in the time dimension. This recurrence is central to SSMs and implicitly constrains how researchers think about scaling them up. The general design pattern is simple, usually stacking $L$ independent SSM layers, each with its own parameters, and letting depth provide representational diversity. However, layers-depth recurrence is completely orthogonal to time recurrence, and further decreases SSMs memory footprint. We show that a looped SSM with $k$ parameters adaptively iterated $M$ times, achieves a performance comparable to a standard SSM with $k\,{\cdot}\,L$ independent parameters, where $M\,{\leq}\,L$. 

Asking the second question is also not trivial, given that (nonlinear) recurrent neural networks (RNNs) are the original roots of SSMs, and RNNs were trained in a sequential fashion. In RNNs, acquisition and training have the same granularity, and this is a property of the incrementally-processed input sequence. However in SSMs, training happens in parallel, on the entire input sequence. Hence, it makes sense to ask if reshaping the input sequence in chunks that are consistent across all tasks considered can improve accuracy? After all, the SSMs see the entire sequence anyway, and consistency might be beneficial. Indeed, we show that concatenating time steps for lower-dimensional sequence elements, or flattening and re-chunking the time steps for high-dimensional elements, can considerably improve the baseline, by enhancing information presentation to the model.

The interaction between input-dimensionality, sequence-length, and model-capacity is still poorly understood in SSMs, and this work sheds some light on this interaction. We show that depth recurrence decreases their memory footprint without any penalty, and input reshaping positively affects SSMs performance. Moreover, the combination of both further improves SSM accuracy. Our experiments were conducted on four representative SSM baselines: LRU \citep{orvieto2023resurrecting}, S5 \citep{smith2023simplified}, LinOSS \citep{rusch2024oscillatory}, and LrcSSM \citep{farsang2025parallelizationnonlinearstatespacemodels}, across six time series benchmarks spanning input dimensionalities between 6 to 63, and sequence lengths between 400 to 50,000. %ppg-dalia is 49,920
Our experimental results show that both extensions lead to consistent benefits.

%Our central empirical finding is that a looped SSM with $k$ parameters iterated $L$ times consistently closely matches a standard SSM with $k \cdot L$ independent parameters. The latter contains the former as a special case, meaning that this performance cannot be explained by greater expressivity and instead points to parameter sharing across depth as a beneficial inductive bias.

%\paragraph{Contributions.} 
\textbf{Our main contributions in this paper are as follows:}
\begin{itemize}
\item \emph{We systematically study depth-recurrence across representative SSM architectures}, showing that recurrence across depth provides a distinct design axis from their native sequence-level recurrence. We evaluate three parameter-sharing strategies under two supervision schemes (final-layer and block-wise loss) across four architectures and six benchmarks.

\item \emph{We demonstrate the use of adaptive inference}, by chosing the number of depth iterations at inference time according to the change between successive predictions, allowing depth-recurrent SSMs to automatically adjust their computational depth without retraining.

\item \emph{We show that depth-recurrent SSMs achieve comparable performance to independently parameterized models} while using substantially fewer parameters. A block with $k$ parameters iterated $L$ times uses as little as $1/L$ of the unique SSM-block parameters of a model with $L$ independent blocks. Parameter sharing across depth is thus a beneficial inductive bias. 

\item \emph{We introduce input reshaping as a complementary design axis for SSMs}, studying how changing the information density (and thus sequence length) presented to the model affects SSMs performance across different tasks with differnt input-dimensionality regimes.

%
%Since the independently parameterized model contains the depth-recurrent model as a special case, this performance cannot be attributed to greater expressivity, pointing to parameter sharing across depth as a beneficial inductive bias.
\end{itemize}

%\section{Depth-Recurrent SSMs and Input Reshaping}

\section{Depth Recurrence in SSMs}
\label{sec:depth_recurrence}

SSMs already employ a recurrent state transition along the sequence dimension. Here, we introduce a second recurrence along layers depth: rather than applying a sequence of independently parameterized blocks, SSM blocks are periodically reused along depth. Thus, depth recurrence is not simply a reduction in parameter count through weight sharing; it turns the depth dimension itself into an iterative computation. Weight sharing follows from this recurrent construction.

Let $\mathbf{z}_0\,{\in}\,\mathbb{R}^{T\times d}$ denote the input-sequence representation, where $T$ is the sequence length, and $d$ is the dimension its elements, and let $f_{\theta_\ell}$ denote an SSM block at depth $\ell$. A depth-$L$ SSM computes:
\begin{equation}
\mathbf{z}_{\ell}
=
f_{\theta_\ell}(\mathbf{z}_{\ell-1}),
\qquad
\ell=1,\ldots,L.
\label{eq:depth_update}
\end{equation}
Depending on the architecture, the time-dependent state recurrence within $f_{\theta_\ell}$ may be linear (LRU and S5), based on forced linear second-order ODEs (LinOSS), or nonlinear (LrcSSM).

\paragraph{Independent layers.}
In a classic depth-$L$ SSM, every depth step has independent parameters
\begin{equation}
\theta_1,\theta_2,\ldots,\theta_L,
\end{equation}
giving a total parameter count of $kL$ when each block contains $k$ parameters.

\paragraph{Periodic Depth Recurrence.}
We define the \emph{recurrence period} $m$ as the number of unique SSM blocks before the depth-wise computation repeats. For $m\,{\in}\,\mathrm{Div}(L)$, the parameters at depth $\ell$ follow
\begin{equation}
\theta_\ell
=
\theta_{((\ell-1)\bmod m)+1},
\qquad
\ell=1,\ldots,L.
\label{eq:periodic_depth}
\end{equation}
The resulting model contains $m$ unique blocks, each applied $r=L/m$ times. As a consequence the model has $km$ parameters. The full depth-recurrence model corresponds to $m\,{=}\,1$, while the independently parameterized model corresponds to $m\,{=}\,L$.

For a fixed depth $L$, the valid periodic configurations are given by the divisors of $L$. When $L=6$,
\begin{equation}
m\in\mathrm{Div}(6)={1,2,3,6},
\end{equation}
which leads to the following depth-recurrent model configurations:
\begin{equation}
m=1 \rightarrow \texttt{AAAAAA, }
m=2 \rightarrow\texttt{ABABAB, }
m=3\rightarrow \texttt{ABCABC, }
m=6 \rightarrow\texttt{ABCDEF}.
\label{eq:periods}
\end{equation}
Each letter denotes an independently parameterized SSM block. Thus, these configurations exhaust the periodic depth-recurrence design space for $L\,{=}\,6$, from a single block recurrently applied across all depth steps to the standard six-block independent model. Patterns whose period does not divide $L$, such as \texttt{ABCDAB}, contain an incomplete final period and lie outside this periodic family.

\begin{figure}
\centering
\includegraphics[width=0.8\linewidth]{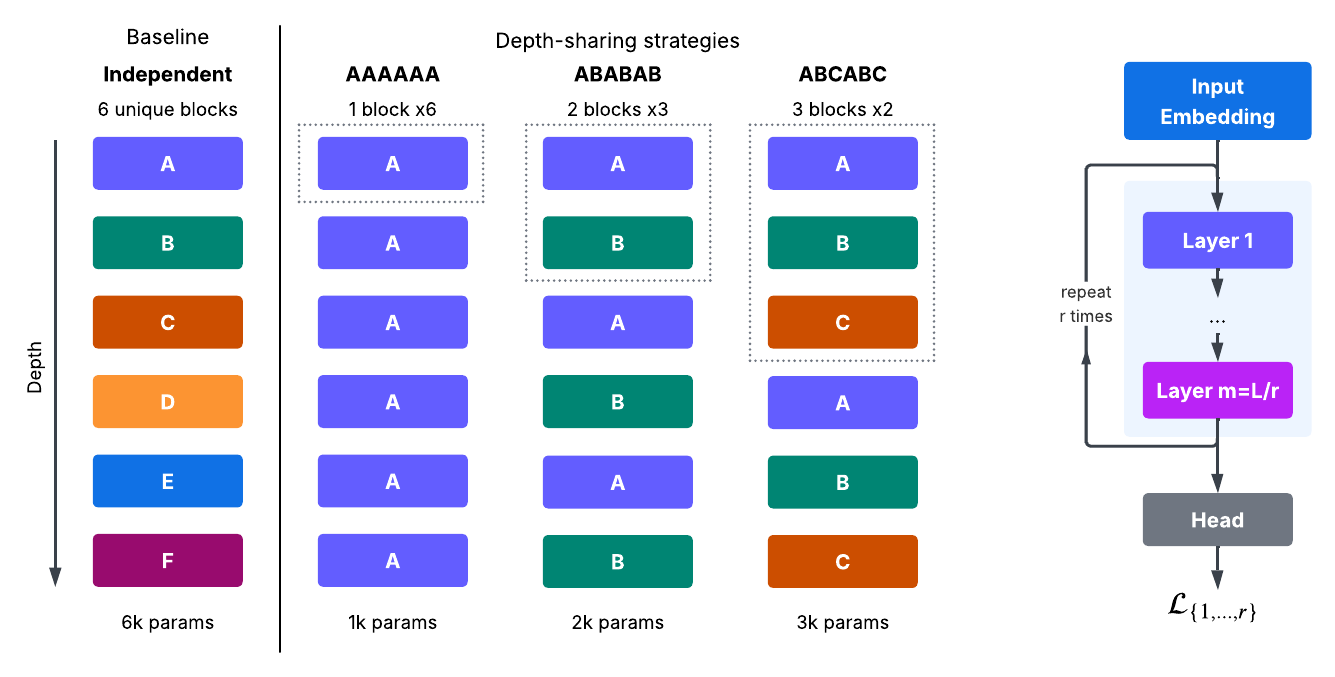}
\caption{Depth-recurrent SSM configurations for $L=6$. The baseline \texttt{ABCDEF} has recurrence period $m=6$, while \texttt{AAAAAA}, \texttt{ABABAB}, and \texttt{ABCABC} correspond to $m=1$, $m=2$, and $m=3$, respectively. We evaluate each depth-recurrent configuration using final or block-wise supervision.}
\label{fig:placeholder}
\end{figure}

\subsection{Supervision and Adaptive Inference}
\label{sec:supervision}

For a recurrence period $m$, let $r=L/m$ denote the number of repetitions and let $\mathbf{h}^{(j)}$ denote the representation after the $j$-th complete period, $j=1,\ldots,r$. Given a prediction head $g_\phi$ and target $y$, we consider the following two supervision strategies.

\paragraph{Final supervision.}
The loss is computed only after the final depth repetition:
\begin{equation}
\mathcal{L}^{\mathrm{final}}
=
\mathcal{L}\left(
g_\phi(\mathbf{h}^{(r)}),y
\right).
\end{equation}

\paragraph{Block-wise supervision.}
The same prediction head is applied after every complete repetition:
\begin{equation}
\mathcal{L}^{\mathrm{block}}
=
\frac{1}{r}
\sum_{j=1}^{r}
\mathcal{L}\left(
g_\phi(\mathbf{h}^{(j)}),y
\right).
\end{equation}
The prediction head $g_\phi$ is shared across all supervision points. Block-wise supervision therefore provides intermediate training signals along the depth-recurrence trajectory without introducing additional prediction heads. The choice of supervision is independent of the recurrence period $m$.

\paragraph{Adaptive inference.}
Depth recurrence also permits the number of iterations to be varied at inference time. Let $\mathbf{p}^{(j)}$ denote the prediction after repetition $j$. We terminate the depth iteration when the relative change between consecutive predictions falls below a threshold $\epsilon$:
\begin{equation}
\frac{
\left|\mathbf{p}^{(j)}-\mathbf{p}^{(j-1)}\right|_2
}{
\left(
\left|\mathbf{p}^{(j)}\right|_2+
\left|\mathbf{p}^{(j-1)}\right|_2
\right)/2
}
<\epsilon.
\label{eq:adaptive_inference}
\end{equation}
The adaptive stopping criterion therefore allows the depth of the inference to adapt, based on the normalized change between two successive predictions.

\subsection{Expressivity of Periodic Depth Recurrence}

\label{sec:theoretical}

For fixed $L$ and per-block parameter budget $k$, let $\mathcal{H}_m$ denote the hypothesis class induced by recurrence period $m$. Given this notation, the independently parameterized model is therefore $\mathcal{H}_L$, while the fully depth-recurrent model is $\mathcal{H}_1$.

\begin{proposition}[Monotone containment under divisibility]
\label{prop:monotone}
Let $m$ and $m'$ be divisors of $L$ such that $m\,{\in}\,\mathrm{Div}(m')$. Then the following holds:
\begin{equation}
\mathcal{H}_m \subseteq \mathcal{H}_{m'}.
\end{equation}
Consequently, for every $m\,{\in}\,\mathrm{Div}(L)$, one has that:
\begin{equation}
\mathcal{H}_1
\subseteq
\mathcal{H}_m
\subseteq
\mathcal{H}_L.
\label{eq:hypothesis_containment}
\end{equation}
\end{proposition}

The containment follows because a period-$m'$ model can reproduce any period-$m$ model, whenever $m\,{\in}\,\mathrm{Div}(m')$, by periodically tying its parameters according to~\Eqref{eq:periodic_depth}. We provide the full proof of this containment in Appendix~\ref{app:theoretical}.

For $L=6$, this gives the following parameterized model containments:
\begin{equation}
\mathcal{H}_1\subseteq\mathcal{H}_2\subseteq\mathcal{H}_6,
\qquad
\mathcal{H}_1\subseteq\mathcal{H}_3\subseteq\mathcal{H}_6,
\end{equation}
However, $\mathcal{H}_2$ and $\mathcal{H}_3$ are not ordered, because neither $2\,{\in}\,\mathrm{Div}(3)$ nor $3\,{\in}\,\mathrm{Div}(2)$. Thus, every depth-recurrent configuration considered in our experiments is contained in the independently parameterized baseline. Matching or outperforming that baseline cannot be therefore explained by a greater expressivity of the depth-recurrent model. Instead, it indicates a beneficial inductive bias which constrains the recurrent computation to use the same parameters across depth.

\section{Input Reshaping in SSMs}
\label{sec:input_reshaping}

Our second design axis modifies the input-sequence (input time-series) representation, independently of the recurrence period $m$. We use a (desirably optimal) concentration factor $c$, to control how many scalar input values are grouped into an input vector at every time step. This trades sequence length for per-step information density. Figure~\ref{fig:reshaping} depicts the two versions described below.

\paragraph{Low-dimensional inputs.}
For an input $\mathbf{X}\in\mathbb{R}^{T\times d}$, we choose $c$ as a multiple of $d$ and concatenate $c/d$ consecutive timesteps along the feature dimension. The resulting sequence has following shape:
\begin{equation}
(T,d)
\longrightarrow
\left(
\left\lceil\frac{Td}{c}\right\rceil,
c
\right),
\qquad c\,{>}\,d,\quad d\,{\in}\,\mathrm{Div}(c),
\label{eq:low_dim_reshape}
\end{equation}
with zero-padding where necessary. This reshaping thus concatenates the complete feature vectors at each original time steps, and shortens the sequence by approximately a factor of $c/d$.

\paragraph{High-dimensional inputs.}
For high-dimensional inputs $\mathbf{X}\in\mathbb{R}^{T\times D}$, we first flatten the joint feature-and-time sequence along its scalar values, and re-chunk it into vectors of width $c$:
\begin{equation}
(T,D)
\longrightarrow
\left(
\left\lceil\frac{TD}{c}\right\rceil,
c
\right)
\qquad c\,{<}\,D.%,\quad c\,{\in}\,\mathrm{Div}(D),
\label{eq:high_dim_reshape}
\end{equation}
Unlike temporal concatenation, this operation deliberately allows feature and temporal identities to mix within a single input vector. Zero-padding is applied when $TD$ is not divisible by $c$.

Depth recurrence and input reshaping therefore modify two distinct aspects of the recurrent SSM model: the depth-recurrence period $m$ controls how computation is repeated across depth, while the concentration factor $c$ controls how information is presented along the sequence (or time) dimension. In the next section, we study their effects both independently and in combination.

\begin{figure}
    \centering
    \includegraphics[width=0.9\linewidth]{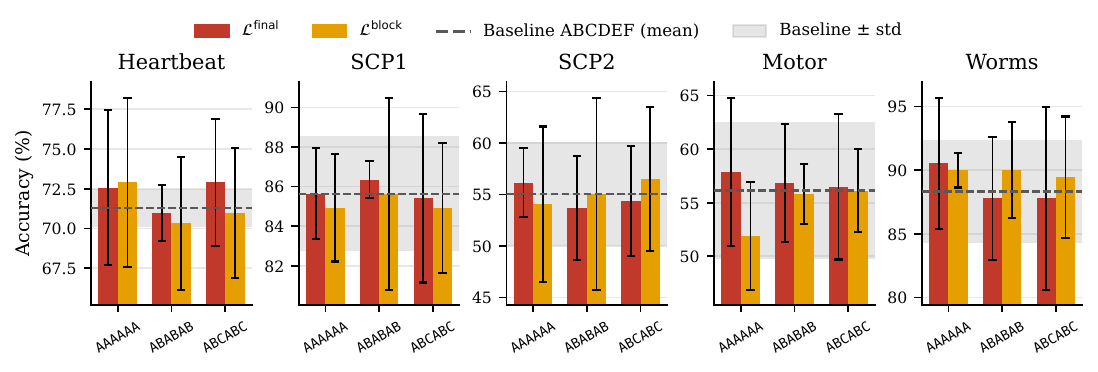}
    \includegraphics[width=0.9\linewidth]{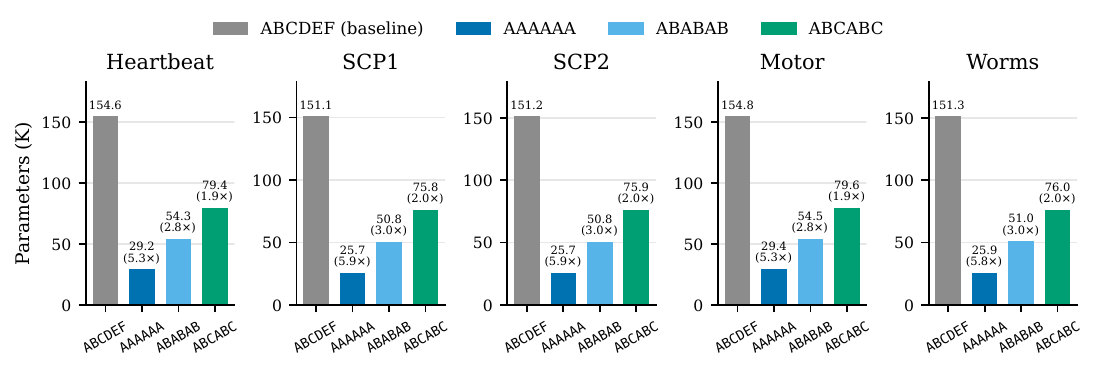}
    \caption{Classification accuracy (\%) across datasets for different LRU
    depth-recurrence configurations, averaged over 5 seeds (top), with the
    corresponding parameter reductions (bottom). Depth-recurrent models achieve
    comparable performance with substantially fewer parameters.}
    \label{fig:looping_acc_param}
\end{figure}

\section{Empirical Results}\label{sec:results}

\subsection{Datasets and Experimental Setup}
\label{sec:expsetup}

We evaluate periodic depth recurrence and input reshaping both on multivariate
time-series classification datasets from the UEA Multivariate Time Series
Classification Archive~\cite{bagnall2018uea} and on the heart-rate
estimation regression task of PPG-DaLiA~\cite{Reiss2019DeepPL}. These benchmarks span a
wide range of sequence lengths and input dimensionalities and are commonly
used to evaluate the performance of sequence and SSM models~\citep{walker2024log, moreno2024rough,
nzoyem2025weight, rusch2024oscillatory, boyer2025learning, pourcel2026learning,
karuvally2026bridging}. We use fixed 70/15/15 train/validation/test splits
generated with the same predefined random seeds. Dataset details are provided
in Table~\ref{tab:dataset_details} in Appendix~\ref{sec:dataset_description}.

We evaluate the following representative SSM models: S5~\citep{smith2023simplified},
LRU~\citep{orvieto2023resurrecting}, LinOSS~\citep{rusch2024oscillatory},
and LrcSSM~\citep{farsang2025parallelizationnonlinearstatespacemodels}.
We fix the architecture across models, and recurrence periods, using state size
64, hidden dimension 64, and depth $L\,{=}\,6$, which gives the periodic recurrence
periods $m\in\{1,2,3,6\}$. We tune only the learning rate over
$\{10^{-3},10^{-4},10^{-5}\}$. Further dataset, model, and training details
are provided in Appendix~\ref{sec:dataset_models_training}.

\subsection{Depth Recurrence}\label{sec:looped_ssm}
% (results without reshaping)
% Here, we want to show that we can save parameter count and preserve/enhance prediction accuracy.

We first evaluate depth recurrence without input reshaping, by varying the
recurrence period over $m\,{\in}\,\{1,2,3\}$ and comparing against the independently
parameterized $m\,{=}\,6$ baseline. By Proposition~\ref{prop:monotone}, each
depth-recurrent hypothesis class is contained in that of the independent
baseline. Thus, matching or exceeding the baseline cannot be attributed to
greater expressivity.

Figure~\ref{fig:looping_acc_param} shows that depth-recurrent models frequently
remain close to the independent baseline despite substantial parameter
reductions. In particular, $m\,{=}\,1$ (\texttt{AAAAAA}) reduces the total parameter
count by $5.3$--$5.9\times$ while remaining competitive across several
datasets. Across recurrence periods, neither supervision strategy uniformly
dominates. Complete numerical results across all classification datasets and
models are reported in Table~\ref{tab:SSM_looping} in the Appendix.

To verify that this behavior is not specific to classification, we additionally
evaluate depth recurrence on the PPG-DaLiA regression task.
Figure~\ref{fig:ppg_dalia_bars} shows that all depth-recurrent configurations
match or outperform the accuracy of the corresponding independently-parameterized SSM
baselines, indicating that the observed behavior extends beyond classification tasks.

\begin{figure}
    \centering
    \includegraphics[width=0.8\linewidth]{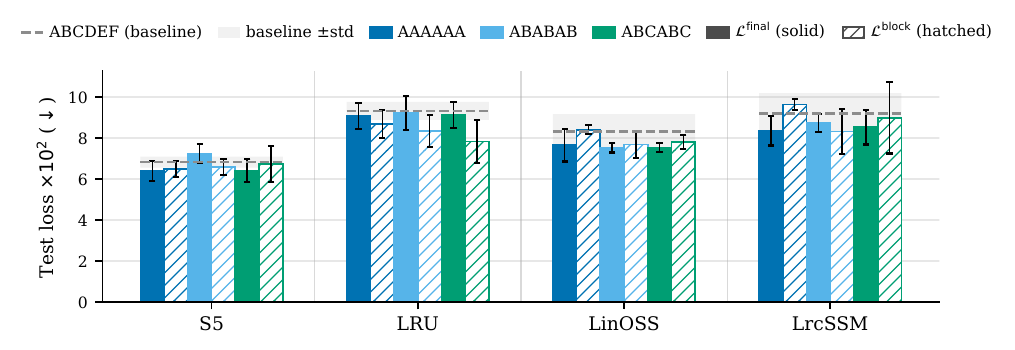}
    \caption{Test loss ($\times10^2$, lower is better) on PPG-DaLiA across S5,
        LRU, LinOSS, and LrcSSM depth-recurrence configurations, averaged over 5
        seeds. Bars show \texttt{AAAAAA}, \texttt{ABABAB}, and \texttt{ABCABC} with
        $\mathcal{L}^{\mathrm{final}}$ (solid) and $\mathcal{L}^{\mathrm{block}}$
        (hatched). The dashed line and band denote the \texttt{ABCDEF} baseline mean
        $\pm1$ std; error bars denote $\pm1$ std.}
    \label{fig:ppg_dalia_bars}
\end{figure}

\subsection{Adaptive Depth at Inference}\label{sec:adaptive}
So far, we have evaluated models at their fixed training depth of $L\,{=}\,6$. We
next investigate whether the same trained model can adapt its computational
depth at inference time, considering both fewer than six recurrent applications
and extrapolation beyond the training depth, up to a maximum of 12
applications. We evaluate adaptive depth for the two most parameter-efficient
depth-recurrent configurations, \texttt{AAAAAA} and \texttt{ABABAB}, using the
LRU models. The results are reported in Figure~\ref{fig:aaaaaa_early_stopping}, for
\texttt{AAAAAA}, and in Table~\ref{tab:ababab_early_stopping} of the Appendix, for
\texttt{ABABAB}.

At inference time, we repeatedly apply the trained recurrent block and stop
when the relative change between consecutive predictions falls below the
threshold defined in Section~\ref{sec:supervision}, or when the maximum number
of applications is reached. It is important to note that this relatively simple 
stopping criterion does not imply
convergence of the hidden states or predictions. It merely detects when
successive predictions differ by less than the prescribed stopping threshold.

For the \texttt{AAAAAA} model, the stopping criterion typically selects
between 2 and 5 repetitions, compared with the fixed 6 repetitions used
during training, while maintaining classification performance. Since fewer
repetitions require less computation, this results in a
substantial reduction in GFLOPs, as shown in
Figure~\ref{fig:aaaaaa_early_stopping}. These results show that depth-recurrent
SSMs support adaptive computation at inference time: the same trained model can
use a variable number of recurrent applications based on the relative change
between successive predictions, without requiring retraining for each depth.

\begin{figure}[b]
    \centering
    \includegraphics[width=0.95\linewidth]{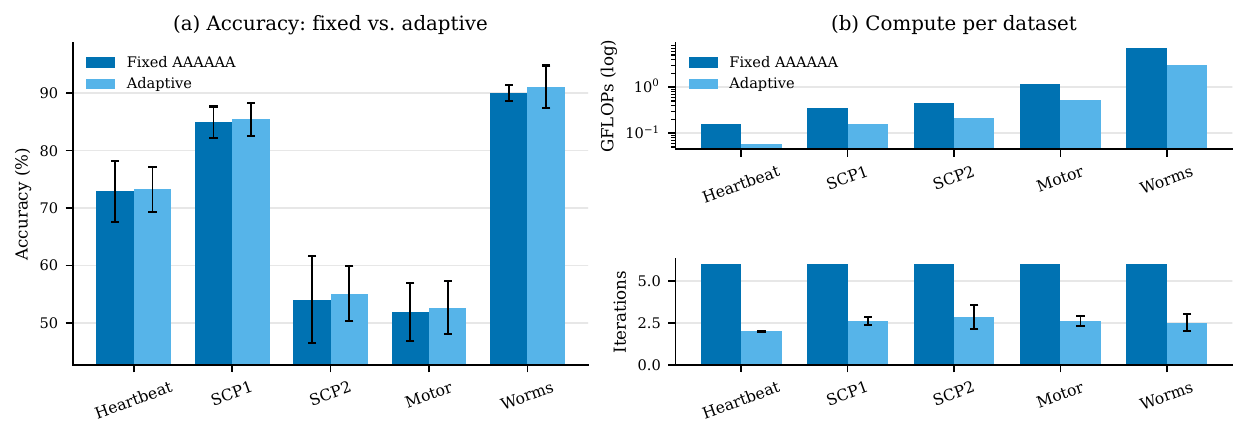}
    \caption{Adaptive inference versus fixed 6-iteration \texttt{AAAAAA} across five benchmarks. (a) Adaptive inference maintains comparable performance. (b) It reduces GFLOPs by typically stopping after 2-3 iterations (mean $\pm$ std over five seeds), compared with the fixed budget of six.}
    \label{fig:aaaaaa_early_stopping}
\end{figure}

\subsection{Input Reshaping}
% (effect of the input reshaping only)

% unified input representations

% instead of tuning the architecture hyperparameters, better input representations help 

We next isolate the effect of input reshaping within the baseline architecture,
without any depth recurrence. This provides a complementary way to adapt the input
representation, without tuning architectural dimensions such as the hidden size,
state size, or number of layers.

Figure~\ref{fig:input_reshape_heartbeat_motor} compares the original
\texttt{ABCDEF} baseline (grey dashed line) with the corresponding reshaped-input
\texttt{ABCDEF} baseline (purple solid line). Note that input reshaping alone, 
considerably improves the mean of several models, most notably for LRUs on the 
Heartbeat dataset ($71.3\%\rightarrow75.5\%$), and for S5 on the MotorImagery 
dataset ($54.4\%\rightarrow61.8\%$).

\subsection{Combining Input Reshaping and Depth Recurrence}
We next apply depth recurrence on top of the reshaped inputs. Using the same
comparison in Figure~\ref{fig:input_reshape_heartbeat_motor}, the
depth-recurrent configurations with reshaped inputs consistently improve over
their counterparts using the original inputs. On Heartbeat, LRU reaches
$76.1\%$ with the combined approach, compared with $72.9\%$ for the
corresponding depth-recurrent configurations without reshaping. On
MotorImagery, S5 and LrcSSM reach $60.7\%$ and $60.7\%$, respectively,
compared with $55.8\%$ and $58.2\%$, respectively, with depth recurrence alone. 
These results show that the benefits of input reshaping are retained when 
combined with depth recurrence.

Complete numerical results for this experiment are reported in
Table~\ref{tab:SSM_looped_reshaped}. Medium-dimensional datasets
(SCP1, SCP2, and Worms) show no substantial gains from reshaping, suggesting
that reshaping is most beneficial when the original input representation is
poorly matched to the fixed architecture (state-size and hidden dimension). 
%Results for these datasets are
% reported in Tables~\ref{tab:input_reshape_medium}
% and~\ref{tab:recursion_reshaping_medium} in the Appendix. 
Additional input
reshaping experiments are provided in
Appendix~\ref{sec:additional_inputreshaping}.

% These findings suggest that input reshaping is a complementary and often overlooked axis, where one can gain considerably better performance, without modifying the underlying architecture at all.

\begin{figure}
    \centering
    \begin{subfigure}{0.49\linewidth}
        \centering
        \includegraphics[width=\linewidth]{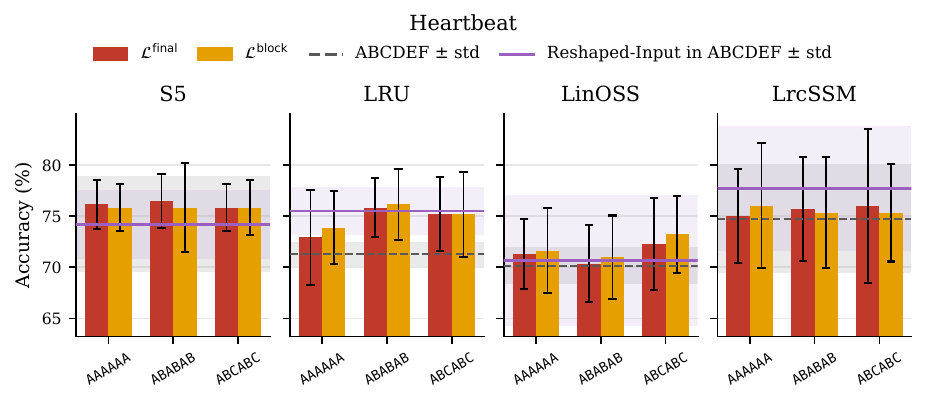}
        \caption{Heartbeat}
        \label{fig:input_reshape_heartbeat}
    \end{subfigure}
    \hfill
    \begin{subfigure}{0.49\linewidth}
        \centering
        \includegraphics[width=\linewidth]{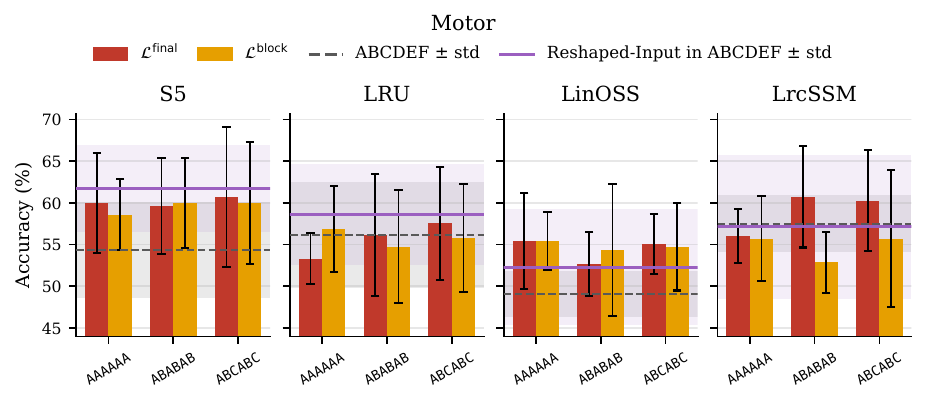}
        \caption{MotorImagery}
        \label{fig:input_reshape_motor}
    \end{subfigure}

    \caption{Classification accuracy (\%) on Heartbeat and MotorImagery with input
    reshaping and depth recurrence. Bars show \texttt{AAAAAA}, \texttt{ABABAB},
    and \texttt{ABCABC} with $\mathcal{L}^{\mathrm{final}}$ (red) and
    $\mathcal{L}^{\mathrm{block}}$ (orange). Lines show the original (grey dashed)
    and reshaped-input (purple) \texttt{ABCDEF} baselines. Bands and error bars
    denote $\pm1$ std over 5 seeds.}
    \label{fig:input_reshape_heartbeat_motor}
    \vspace*{-2ex}
\end{figure}

Input reshaping can also be applied during online inference, with a corresponding change in the input-output frequency. For lower-dimensional inputs, concatenating consecutive timesteps requires buffering multiple observations before forming a model input, introducing a delay in the timesteps. For high-dimensional inputs, rechunking can instead map a single original timestep
to multiple model inputs, from which predictions can be produced at the
corresponding model steps.

Taken together, these results suggest that reshaping and  depth-recursion are complementary: reshaping lifts the baseline by improving how information is presented to the 
model, while depth-recursion can extract additional value from the resulting representation with a reduced parameter count.

\subsection{Ablations on Parameter-Matched Setups and Damped Iterative Updates}
We further investigate whether the benefits of depth recurrence can be explained
by the reduced parameter count alone. We focus on LRUs, which showed the most
consistent performance under depth recurrence. Figure~\ref{fig:lru_param_matched_bars_highlighted}
in Appendix~\ref{sec:parameter_matched} shows that depth-recurrent models can
substantially outperform parameter-matched single-block baselines.

We also consider a more conservative update rule based on damped iterative
updates following the Krasnosel'skii-Mann iteration~\citep{krasnosel1955two,
mann1953mean}. This update rule preserves part of the previous state while incorporating
the newly computed update with a smaller weight. We find that damping does not
lead to substantially different results. The full results and analysis the Krasnosel'skii-Mann iteration are provided in Appendix~\ref{sec:damped_iterative_updates}.

\section{Discussion and Future Work}

\textbf{On the mismatch between model input and data structure.}
A recurring observation motivating this work is that standard SSM pipelines
typically process one input vector per time step, regardless of the
dimensionality of that vector. For low-dimensional inputs (small $d$), each
time step contains relatively little information, resulting in long sequences
of small input vectors. For high-dimensional inputs (large $D$), each time
step instead contains a large number of features that must be processed
jointly. Input reshaping provides a simple way to modify this tradeoff through
the concentration factor $c$, without changing the underlying architecture.
Our results show clear gains on the high-dimensional Heartbeat and
MotorImagery datasets, while the medium-dimensional datasets benefit less.
This suggests that the effectiveness of reshaping strongly depends on the mismatch between
the state size of the model and the input dimension of the sequence elements.

\textbf{On the benefits of depth recurrence.}
Across the main experiments, depth-recurrent SSMs frequently remain competitive
with independently parameterized baselines despite using substantially fewer
parameters. By Proposition~\ref{prop:monotone}, the corresponding
depth-recurrent hypothesis classes are contained in that of the independent
baseline, so these results cannot be explained by greater expressivity.
Moreover, the parameter-matched ablation in
Appendix~\ref{sec:parameter_matched} shows that the observed gains cannot be
explained by parameter reduction alone: depth-recurrent LRUs can outperform
both the full independent model and parameter-matched single-block baselines.
Together, these results suggest that a repeated computation with shared parameters
across depth, is most likely a very useful inductive bias, rather than simply a 
parameter-saving mechanism.

\textbf{Adaptive computation.}
Depth recurrence also enables the amount of computation to vary at inference
time without retraining. Our adaptive inference experiments show that the
prediction-change stopping criterion often terminates before the six
applications used during training, while maintaining comparable performance
and reducing GFLOPs. Importantly, this stopping criterion measures only the
relative change between successive predictions and does not imply convergence.
This provides a practical additional benefit of treating depth as an iterative
computation.

\textbf{Architecture-specific behavior.}
The gains from depth recurrence are not uniform across architectures. LRU
benefits most consistently, while LinOSS and LrcSSM show a more mixed picture,
including drops on some benchmarks. We hypothesize that this reflects an
interaction between the recurrence structure of each architecture and the
inductive bias introduced by repeated parameter sharing, but a principled
explanation remains open. Understanding when and why this interaction is
favorable is an important direction for future work, as we note in
Remark~\ref{app:depth_iteration_theory}.

\textbf{Limitations and future work.}
Our experiments focused on time-series classification and regression tasks, and
extending the analysis to other sequence modeling tasks remains an important
next step. We also fix the training depth to $L\,{=}\,6$ and treat the concentration
factor $c$ as a tuned scalar hyperparameter. Future work could characterize
when depth recurrence is beneficial, investigate how the appropriate
concentration factor depends on the structure of the input, and study training
directly with variable depth. Finally, reweighting the intermediate loss terms
in $\mathcal{L}^{\mathrm{block}}$ could provide a curriculum-like training
signal and is another direction worth exploring.

\section{Related Work}
\paragraph{State Space Models.}
Structured State Space Models (SSMs) emerged as a scalable alternative to attention-based Transformers for long-range sequence modeling. S4 \citep{gu2022efficientlymodelinglongsequences} introduced a parameterization of linear SSMs via the HiPPO framework, achieving strong performance on the Long-Range-Arena benchmark. S5 \citep{smith2023simplified} streamlined this design by replacing the independent SISO systems with a single MIMO SSM solved via parallel scan. The Linear Recurrent Unit (LRU) \citep{orvieto2023resurrecting} demonstrated that careful initialization and normalization of a simple diagonal linear RNN, is sufficient to match deep SSMs accuracy, without any continuous-time or HiPPO motivation. LinOSS \citep{rusch2024oscillatory} used forced second-order linear ODE dynamics with strong performance, inspired by cortical oscillations, and
%to achieve strong performance on time series tasks. 
%Liquid-S4 \citep{} and Mamba \citep{gu2023mamba} introduced input-dependent SSM parameters via a selective scan mechanism, bringing content-based reasoning to the SSM family. 
most recently, LrcSSMs \citep{farsang2025parallelizationnonlinearstatespacemodels} extended SSMs to biologically plausible nonlinear recurrent dynamics, forcing a diagonal Jacobian structure allowing parallel training while retaining nonlinear expressivity.  

The main focus of all these previous works, was on the design of increasingly powerful individual SSM layers, stacked independently. In this work instead, we study what happens when these layers are reused across depth. There is no previous work on SSMs systematically addresing this topic.

%A parallel line of work on transformers has recently challenged the analogous assumption made for the transformer architecture. Looped transformers \citep{dehghani2018universal, geiping2025scaling, zhu2025scalinglatentreasoninglooped} reuse the same block repeatedly across depth, showing that a model with $k$ parameters iterated $L$ times can match a standard model with $k \cdot L$ independent parameters. This finding has driven a reappraisal of what depth actually contributes to transformer computation, shifting the narrative from \emph{more parameters} to \emph{more iterative refinement}. Related ideas have also recently appeared in SSMs, where Mamba-2 blocks are iterated in depth toward a fixed point to construct implicit models \citep{schone2025implicit}. However, depth-recurrence as an explicit architectural design choice has not been systematically explored across SSM architectures. SSMs are already recurrent, so adding another axis of recurrence may at first seem redundant. We argue the opposite: \emph{depth-recurrence} is distinct from sequence-recurrence, and this distinction makes depth-recurrence an underexplored source of inductive bias for SSMs.

\vspace*{-1ex}\paragraph{Depth Recurrence.}
In feed-forward neural networks such as transformers, the use of layer-depth recurrence occurred as a necessity, to reducing the sheer size of these networks. \citep{dehghani2018universal} first showed how to apply the same block repeatedly and adaptively at computation time. This work was considerably revived in test-time scaling.
\citep{zhu2025scalinglatentreasoninglooped} scaled this approach 
%to 7.7 trillion training tokens, 
to 2.6B parameters matching a
8B standard transformer, and \citet{saunshi2025reasoning}  
showed that looped transformers can match much deeper non-looped models on reasoning tasks. 
\citet{geiping2025scaling} trained a 3.5B-parameter model with a fixed recurrent block iterated up to 32 times at inference, showing that effective depth can be scaled independently of parameter count. 
\citep{prairie2026parcae} addressed training instability in looped models, established scaling laws for stable looped architectures, and
%
%\citet{saunshi2025reasoning} formally connected looping to the generation of latent thoughts, demonstrating that looped transformers can match much deeper non-looped models on reasoning tasks. 
%
\citep{pappone2025two, blayney2026mechanistic} studied the geometry of iterates inside looped blocks, finding that updates perform fine-grained latent refinements rather than coarse directional pushes. 

In time-recurrent neural networks such as SSMs, depth recurrence was largely ignored. Here, we thoroughly study adaptive depth-recurrence explicitly, as an architectural design choice for SSMs, by systematically comparing parameter-shared and independently parameterized depth across multiple SSM families. On a related but different line of work, \citet{schone2025implicit} iterates Mamba-2 blocks in depth towards a fixed-point, in order to obtain an implicit language model. Recurrent Highway Networks~\cite{zilly2017recurrent}, repeatedly apply a transformation within a single recurrent transition, inside a layer. In contrast, our approach replaces a stack of independent SSM layers with repeated applications of the same SSM block, introducing recurrence along the depth dimension.

\vspace*{-1ex}\paragraph{Input Reshaping.}
Reshaping input data before feeding it into a sequence model, can significantly affect model performance. However, this topic has remained quite underexplored relative to architectural design. If the input is an image, partitioning this into patches is not only necessary, but also relatively intuitive, as the image has no time dimension. For example, ViT \citep{dosovitskiy2020image} partitions images into patches that mix spatial positions within each token. Subsequent work further explores varying patch sizes and hierarchical structures \citep{beyer2023flexivit, liu2021swin}. 

If the input is a time-series however, the most natural thing to do is to use as patches the vectors available at each moment of time. This is precisely the approach used by all SSMs so far. In contrast, our novel reshaping strategy is much less obvious from a time-series perspective. It concatenates low-dimensional sequence vectors across time-steps, and flattens and re-partitions
high-dimensional sequence vectors across time-steps. This way it achieves an optimal vector size
per sequence element, that is presented to the model. This interaction between original vector size, sequence length, and optimal vector size, has not been studied in SSMs so far. Related ideas have been explored for time series in PatchTST~\citep{nie2022time}, which aggregates consecutive time-step patches into a single patch. However, PatchTST did not consider splitting a time step in several ones, with according patches. Channel-mixing versus channel-independence has been studied as an architectural choice in forecasting models \citep{han2024capacity}, but this is less related to our current work.

%\vspace{-2ex}
\section{Conclusion}
%\vspace{-2ex}
We study two complementary design axes for SSMs: recurrence across depth and the structure of the input representation. Across multiple SSM architectures, depth recurrence can preserve or improve performance while substantially reducing parameter count, despite operating within a hypothesis class contained in that of the independently parameterized baseline. It also enables adaptive inference, allowing the same trained model to use fewer depth iterations when successive predictions change only marginally. Independently, input reshaping can improve performance without modifying the underlying architecture, and combining the two can provide further gains. Together, these results show that the way an SSM reuses computation across depth, and how its input sequence is represented, are important design choices alongside the recurrent dynamics of the SSM itself.

\subsubsection*{Acknowledgments}
% Use unnumbered third level headings for the acknowledgments. All
% acknowledgments, including those to funding agencies, go at the end of the paper.
Research was sponsored by the Department of the Air Force Artificial Intelligence Accelerator and was accomplished under Cooperative Agreement Number FA8750-19-2-1000. The views and conclusions contained in this document are those of the authors and should not be interpreted as representing the official policies, either expressed or implied, of the Department of the Air Force or the U.S. Government. The U.S. Government is authorized to reproduce and distribute reprints for Government purposes notwithstanding any copyright notation herein.

\bibliography{references}
\bibliographystyle{ICLR2027/iclr2027_conference}

\newpage
\appendix
% \section{Appendix}

\section{Technical Appendices and Supplementary Material}
% Technical appendices with additional results, figures, graphs, and proofs may be submitted with the paper submission before the full submission deadline (see above). You can upload a ZIP file for videos or code, but do not upload a separate PDF file for the appendix. There is no page limit for the technical appendices. 

% Note: Think of the appendix as ``optional reading'' for reviewers. The paper must be able to stand alone without the appendix; for example, adding critical experiments that support the main claims to an appendix is inappropriate. 
\subsection{Theoretical Analysis}\label{app:theoretical}

We fix a depth $L$ and a per-block parameter budget $k$. For each recurrence
period $m$ dividing $L$, let $\mathcal{H}_m$ denote the hypothesis class of
depth-$L$ models with $m$ unique blocks repeated periodically across depth.
Thus, $\mathcal{H}_1$ corresponds to full depth recurrence
(\texttt{AAAAAA} for $L=6$), while $\mathcal{H}_L$ corresponds to the
independently parameterized model (\texttt{ABCDEF} for $L=6$).

\begin{proposition}[Monotone containment under divisibility]
\label{prop:monotone}
Let $m$ and $m'$ be divisors of $L$ with $m \mid m'$. Then
\begin{equation}
    \mathcal{H}_m \subseteq \mathcal{H}_{m'}.
\end{equation}
In particular, for any divisor $m$ of $L$,
\begin{equation}
    \mathcal{H}_1 \subseteq \mathcal{H}_m \subseteq \mathcal{H}_L.
\end{equation}
\end{proposition}

\begin{proof}
Since $m \mid m'$, write $m'=s m$ for some integer $s\geq1$. Any model in
$\mathcal{H}_m$ has $m$ unique parameter sets
$(\theta_1,\ldots,\theta_m)$ repeated periodically across the $L$ depth
positions. The same model is recovered in $\mathcal{H}_{m'}$ by setting
\begin{equation}
    \theta'_j
    =
    \theta_{((j-1)\bmod m)+1},
    \qquad j=1,\ldots,m'.
\end{equation}
The resulting period-$m'$ model therefore reproduces the same period-$m$
parameter pattern across all $L$ depth positions.
\end{proof}

\begin{remark}
The divisibility condition $m\mid m'$ is important. For $L=6$, the recurrence
periods $m=2$ (\texttt{ABABAB}) and $m=3$ (\texttt{ABCABC}) satisfy neither
$2\mid3$ nor $3\mid2$, and their hypothesis classes are therefore not ordered
by Proposition~\ref{prop:monotone}. For the configurations considered in our
experiments,
\begin{equation}
    \mathcal{H}_1 \subseteq \mathcal{H}_2 \subseteq \mathcal{H}_6,
    \qquad
    \mathcal{H}_1 \subseteq \mathcal{H}_3 \subseteq \mathcal{H}_6.
\end{equation}
Here, $\mathcal{H}_1$, $\mathcal{H}_2$, $\mathcal{H}_3$, and
$\mathcal{H}_6$ correspond to \texttt{AAAAAA}, \texttt{ABABAB},
\texttt{ABCABC}, and \texttt{ABCDEF}, respectively.
\end{remark}

\begin{remark}[Expressivity and inductive bias]
\label{cor:inductive_bias}
By Proposition~\ref{prop:monotone}, every depth-recurrent configuration
considered here is contained in the independently parameterized hypothesis
class $\mathcal{H}_L$. Consequently, matching or exceeding the performance of
the independent baseline cannot be attributed to greater expressivity. The
empirical advantage must therefore arise despite the additional parameter
constraints imposed by depth recurrence, suggesting that repeated computation
with shared parameters provides a useful inductive bias.
\end{remark}

\begin{remark}[Gradient aggregation under parameter sharing]
For a model with recurrence period $m$, let $r=L/m$ denote the number of
repetitions. Because the same parameters are reused across repetitions,
gradients from multiple depth positions accumulate into updates of the same
parameters. Under block-wise supervision $\mathcal{L}^{\mathrm{block}}$, the
gradient with respect to the parameters $\theta$ of a shared block takes the
form
\begin{equation}
    \frac{\partial \mathcal{L}^{\mathrm{block}}}{\partial \theta}
    =
    \frac{1}{r}
    \sum_{j=1}^{r}
    \frac{\partial
    \mathcal{L}\!\left(g_\phi(h^{(j)}),y\right)}
    {\partial \theta},
\end{equation}
where $h^{(j)}$ denotes the representation after the $j$-th repetition. Under
final supervision $\mathcal{L}^{\mathrm{final}}$, the loss is applied only
after the final repetition, while gradients propagate through all repetitions
of the shared parameters. Thus, the two objectives provide different
supervision signals while retaining parameter sharing across depth.
\end{remark}

\begin{remark}[Depth iteration and nonlinearity]
\label{app:depth_iteration_theory}
Although the state recurrences of LRU, S5, and LinOSS are linear, the complete
SSM blocks are nonlinear due to the nonlinear operations surrounding the state
recurrence. Depth recurrence therefore iterates nonlinear block maps, whose
behavior cannot in general be characterized by the spectral properties of the
state-transition matrix alone. A theoretical characterization of when and why
such repeated computation is beneficial remains an open question.
\end{remark}

\subsection{Datasets and Experimental Setup}

\subsubsection{Dataset Details}\label{sec:dataset_description}

The UEA classification benchmarks used in this work are
\textit{EigenWorms} (Worms)~\citep{yemini2013database},
\textit{SelfRegulationSCP1} (SCP1) and \textit{SelfRegulationSCP2}
(SCP2)~\citep{birbaumer1999spelling}, \textit{MotorImagery}
(Motor)~\citep{lal2004methods}, and
\textit{Heartbeat}~\citep{goldberger2000physiobank}. Their sequence lengths
and input dimensionalities are summarized in
Table~\ref{tab:dataset_details}.

\begin{table}[h]
    \centering
    \caption{PPG-DaLiA regression~\citep{Reiss2019DeepPL} and UEA multivariate time-series classification datasets
    \citep{bagnall2018uea} used in our experiments, ordered by input
    dimensionality.}
    \label{tab:dataset_details}
    \begin{tabular}{lcccccc}
    \toprule
          & PPG-DaLiA & Worms & SCP1 & SCP2 & Heartbeat & Motor  \\
    \midrule
      Sequence length $T$ & 49,920 & 17,984 & 896 & 1,152 & 405 & 3,000 \\
      Input dimension  & 6   & 6      & 6   & 7     & 61  & 63    \\
    \bottomrule
    \end{tabular}
\end{table}

Relative to the fixed hidden and state dimensions of 64 used throughout our
experiments, we categorize PPG-DaLiA, Worms, SCP1, and SCP2 as medium-dimensional inputs,
and Heartbeat and Motor as high-dimensional inputs. We denote their input
dimensions by $d$ and $D$, respectively.

\subsubsection{Models and Training Setup}\label{sec:dataset_models_training}

Previous SSM evaluations have used extensive architecture-specific
hyperparameter searches, varying the number of layers, state size, and hidden
dimension~\citep{walkerlog, rusch2024oscillatory,
farsang2025parallelizationnonlinearstatespacemodels}. In contrast, we fix the
state size and hidden dimension to 64 and the depth to $L=6$ across all models.
This controlled setup isolates the effect of depth recurrence from changes in
model capacity and allows input reshaping to modify the input representation
without altering the underlying architecture.

We tune only the learning rate over
$\{10^{-3},10^{-4},10^{-5}\}$. Our implementation builds on the LinOSS
codebase~\citep{rusch2024oscillatory}, which includes S5, LRU, and LinOSS,
and the LrcSSM codebase~\citep{farsang2025parallelizationnonlinearstatespacemodels},
extending them with depth recurrence and input reshaping.

Training was conducted on NVIDIA A40 and A100 GPUs with 40 GB and 80 GB of
memory, respectively. Each training run completed in under 20 minutes,
depending on the dataset and model, with early stopping applied.

\subsection{Additional Depth-Recurrence Results}

\subsubsection{Complete Numerical Results}\label{sec:app_numerical}

Table~\ref{tab:SSM_looping} reports the complete numerical results across all
benchmarks. Depth-recurrent SSMs match or exceed the baseline mean in the majority of settings (19 of 24 for the fully shared configuration), while showing lower accuracy in a few cases. More broadly, across all 144 depth-recurrent configurations, 132 (91.7\%) either achieve a numerically better mean than the independently parameterized baseline or remain within one standard deviation of its mean.

\begin{table}
    \centering
    \caption{Performance across all benchmarks for the independent
    \texttt{ABCDEF} baseline and the depth-recurrent \texttt{AAAAAA},
    \texttt{ABABAB}, and \texttt{ABCABC} configurations under
    $\mathcal{L}^{\mathrm{final}}$ and $\mathcal{L}^{\mathrm{block}}$.
    Classification tasks report accuracy (\%, higher is better), while
    PPG-DaLiA reports test loss ($\times10^2$, lower is better).
    Results are mean $\pm$ std over 5 seeds; bold indicates an improved mean
    over the corresponding \texttt{ABCDEF} baseline.}
    \label{tab:SSM_looping}
    \resizebox{\textwidth}{!}{%
    \begin{tabular}{c|c|c|cccccc}
    \toprule
     &  & Baseline & \multicolumn{6}{c}{Depth recurrence}\\
    Dataset & Model & \texttt{ABCDEF} & \multicolumn{2}{c}{\texttt{AAAAAA}} & \multicolumn{2}{c}{\texttt{ABABAB}} & \multicolumn{2}{c}{\texttt{ABCABC}}\\
    && $\mathcal{L}^{\text{final}}$& $\mathcal{L}^{\text{final}}$ & $\mathcal{L}^{\text{block}}$ & $\mathcal{L}^{\text{final}}$ & $\mathcal{L}^{\text{block}}$ & $\mathcal{L}^{\text{final}}$ & $\mathcal{L}^{\text{block}}$ \\
    \midrule
    \multirow{4}{*}{Heartbeat ($\uparrow$)}
        & S5  & $74.19 \pm 4.67$  & $\mathbf{75.16 \pm 3.90}$ & $72.90 \pm 5.24$ & $73.55 \pm 2.99$ & $73.87 \pm 4.9$ & $\mathbf{76.13 \pm 3.29}$ & $73.87 \pm 1.58$\\
        & LRU & $71.29 \pm 1.21$ & $\mathbf{72.58 \pm 4.89}$ & $\mathbf{72.90 \pm 5.34}$ & $70.97 \pm 1.77$ & $70.32 \pm 4.16$ & $\mathbf{72.90 \pm 4.00}$ & $70.97 \pm 4.08$\\
        & LinOSS & $70.16 \pm 1.80$ & $\mathbf{73.23 \pm 4.85}$ & $\mathbf{70.32 \pm 2.62}$ & $\mathbf{72.58 \pm 4.45}$  & $\mathbf{73.87 \pm 4.49}$  & $\mathbf{73.55 \pm 3.16}$ & $\mathbf{72.58 \pm 5.20}$\\
        & LrcSSM & $74.67 \pm 5.2$ & $\mathbf{78.33 \pm 5.00}$ & $\mathbf{75.67 \pm 6.19}$ & $\mathbf{76.33 \pm 4.77}$ & $\mathbf{77.00 \pm 5.58}$ & $\mathbf{75.67 \pm 3.46}$ & $\mathbf{76.67 \pm 5.00}$\\
    % \cmidrule{1-9}
    \midrule
    \multirow{4}{*}{SCP1 ($\uparrow$)}
        & S5  & $90.12 \pm 2.05$ & $88.94 \pm 1.91$ & $86.59 \pm 4.05$ & $89.65 \pm 2.28$ & $88.71 \pm 2.64$ & $88.24 \pm 1.97$ & $87.29 \pm 2.62$\\
        & LRU & $85.65 \pm 2.92$ & $\mathbf{85.65 \pm 2.28}$ & $84.94 \pm 2.72$  & $\mathbf{86.35 \pm 0.94}$ & $\mathbf{85.65 \pm 4.85}$ & $85.41 \pm 4.25$ & $84.94 \pm 3.28$\\
        &LinOSS & $85.88 \pm 4.27$ & $85.41 \pm 1.20$ & $\mathbf{86.12 \pm 3.82}$ & $83.76 \pm 4.36$ & $84.47 \pm 1.37$ & $85.18 \pm 0.12$ & $83.76 \pm 3.89$\\
        & LrcSSM & $84.52 \pm 4.29$ & $\mathbf{86.19 \pm 2.99}$ & $\mathbf{84.52 \pm 3.57}$ & $84.05 \pm 2.87$ & $\mathbf{84.52 \pm 3.26}$ & $\mathbf{85.71 \pm 4.61}$ & $\mathbf{85.00 \pm4.95}$  \\
    % \cmidrule{1-9}
    \midrule
    \multirow{4}{*}{SCP2 ($\uparrow$)}
        & S5  & $54.74 \pm 5.25$ & $\mathbf{57.54 \pm 3.40}$ & $\mathbf{57.89 \pm 7.36}$ & $\mathbf{55.09 \pm 6.43}$  & $\mathbf{55.09 \pm 7.74}$ & $\mathbf{54.74 \pm 4.06}$ & $\mathbf{56.14 \pm 4.57}$\\
        & LRU & $55.09 \pm 5.04$ & $\mathbf{56.14 \pm 3.33}$ & $54.04 \pm 7.56$ & $53.68 \pm 5.04$ & $55.09 \pm 9.32$  & $54.39 \pm 5.32$ & $\mathbf{56.49 \pm 6.96}$\\
        & LinOSS & $54.39 \pm 6.37$ & $51.58 \pm 4.09$ & $\mathbf{54.74 \pm 7.48}$ & $\mathbf{54.39 \pm 3.51}$ & $49.47 \pm 6.22$ & $\mathbf{55.09 \pm 6.89}$ & $\mathbf{55.79 \pm 2.05}$\\
        & LrcSSM & $57.14 \pm 5.79$ & $50.71 \pm 6.75$ & $50.89 \pm 2.31$ & $54.64 \pm 9.99$  & $55.00 \pm 6.96$ & $56.43 \pm 6.00$ & $50.36 \pm 8.51$\\
    % \cmidrule{1-9}
    % \midrule
    % \multirow{4}{*}{Ethanol ($\uparrow$)}
    %     & S5  & $19.24 \pm 6.95$ & $\mathbf{19.75 \pm 7.14}$ & $\mathbf{20.00 \pm 7.48}$ & $\mathbf{19.75 \pm 7.14}$ & $\mathbf{19.24 \pm 6.95}$ & $\mathbf{19.24 \pm 6.95}$ & $\mathbf{19.49 \pm 6.68}$\\
    %     & LRU  & $27.09 \pm 3.63$ & $\mathbf{27.34 \pm 3.81}$ & $\mathbf{27.09 \pm 3.26}$ & $\mathbf{27.34 \pm 3.81}$ & $26.84 \pm 3.53$ & $\mathbf{27.34 \pm 3.81}$ & $26.84 \pm 3.53$\\
    %     & LinOSS & $\mathbf{23.73 \pm 7.82}$  & $19.24 \pm 6.72$ & $19.49 \pm 6.68$ & $18.99 \pm 6.55$ & $19.24 \pm 6.95$ & $19.24 \pm 6.95$ & $19.49 \pm 8.19$\\
    %     & LrcSSM & $\mathbf{36.05 \pm 2.00}$ & $33.95 \pm 4.87$ & $34.21 \pm 5.96$ & $33.68 \pm 4.32$ & $31.58 \pm 5.34$ & $33.42 \pm 4.98$ & $34.21 \pm 4.46$\\
    % % \cmidrule{1-9}
    \midrule
    \multirow{4}{*}{Motor ($\uparrow$)}
        & S5  & $54.39 \pm 5.77$  & $\mathbf{55.09 \pm 3.06}$ & $54.04 \pm 11.12$ & $53.68 \pm 5.16$ & $53.68 \pm 11.13$ & $\mathbf{55.79 \pm 5.91}$ & $\mathbf{54.39 \pm 10.05}$\\
        & LRU  & $56.14 \pm 6.37$ & $\mathbf{57.89 \pm 6.93}$ & $51.93 \pm 5.04$ & $\mathbf{56.84 \pm 5.50}$ & $55.79 \pm 2.81$ & $\mathbf{56.49 \pm 6.79}$ & $\mathbf{56.14 \pm 3.84}$\\
        & LinOSS & $49.12 \pm 2.77$ & $\mathbf{51.93 \pm 6.34}$ & $\mathbf{51.23 \pm 4.76}$ & $\mathbf{56.49 \pm 4.35}$ & $\mathbf{54.04 \pm 2.33}$  & $\mathbf{55.09 \pm 2.11}$ & $\mathbf{53.33 \pm 7.24}$\\
        & LrcSSM & $57.50 \pm 3.43$ & $\mathbf{57.86 \pm 3.91}$ & $53.57 \pm 5.50$  & $56.43 \pm 5.30$ & $\mathbf{58.21 \pm 6.13}$ & $55.71 \pm 2.33$ & $55.71 \pm 6.85$\\
    % \cmidrule{1-9}
    \midrule
    \multirow{4}{*}{Worms($\uparrow$)}
        & S5  & $86.67 \pm 5.10$ & $\mathbf{89.44 \pm 4.08}$ & $\mathbf{87.78 \pm 5.44}$  & $85.56 \pm 3.24$ & $82.78 \pm 6.19$ & $83.89 \pm 5.39$ & $84.44 \pm 5.15$\\
        & LRU & $88.33 \pm 4.08$ & $\mathbf{90.56 \pm 5.15}$ & $\mathbf{90.00 \pm 1.36}$ & $87.78 \pm 4.84$  & $\mathbf{90.00 \pm 3.77}$ & $87.78 \pm 7.16$ & $\mathbf{89.44 \pm 4.78}$ \\
        & LinOSS & $95.83 \pm 3.11$ & $90.00 \pm 6.48$  & $93.89 \pm 4.78$ & $83.33 \pm 5.56 $ & $91.67 \pm 7.24$ & $82.78 \pm 11.3$ & $87.22 \pm 13.90$\\
        & LrcSSM & $83.33 \pm 5.56$ & $\mathbf{85.00 \pm 5.41}$ & $\mathbf{85.56 \pm 3.62}$ & $\mathbf{86.67 \pm 7.19}$ & $81.67 \pm 5.76$ & $\mathbf{83.89 \pm 7.71}$ & $82.22 \pm 8.0$\\
    % \cmidrule{1-9}
    % \cmidrule{1-9}
    \midrule
    \midrule
    \multirow{4}{*}{PPG-DaLiA ($\downarrow$)}
        & S5
        & $6.83 \pm 0.27$
        & $\mathbf{6.40 \pm 0.49}$
        & $\mathbf{6.50 \pm 0.40}$
        & $7.25 \pm 0.48$
        & $\mathbf{6.59 \pm 0.41}$
        & $\mathbf{6.42 \pm 0.55}$
        & $\mathbf{6.73 \pm 0.89}$ \\
        & LRU
        & $9.33 \pm 0.44$
        & $\mathbf{9.08 \pm 0.63}$
        & $\mathbf{8.69 \pm 0.70}$
        & $\mathbf{9.23 \pm 0.84}$
        & $\mathbf{8.34 \pm 0.78}$
        & $\mathbf{9.14 \pm 0.64}$
        & $\mathbf{7.84 \pm 1.04}$ \\
        & LinOSS
        & $8.33 \pm 0.84$
        & $\mathbf{7.65 \pm 0.79}$
        & $\mathbf{8.41 \pm 0.22}$
        & $\mathbf{7.54 \pm 0.24}$
        & $\mathbf{7.69 \pm 0.67}$
        & $\mathbf{7.54 \pm 0.22}$
        & $\mathbf{7.81 \pm 0.34}$ \\
        & LrcSSM
        & $9.20 \pm 1.00$
        & $\mathbf{8.35 \pm 0.71}$
        & $9.64 \pm 0.28$
        & $\mathbf{8.75 \pm 0.43}$
        & $\mathbf{8.33 \pm 1.09}$
        & $\mathbf{8.54 \pm 0.85}$
        & $\mathbf{8.99 \pm 1.74}$ \\
    \bottomrule
    \end{tabular}
    }
% \vspace*{-3ex}
\end{table}

\subsubsection{Parameter and Runtime Analysis}\label{sec:param_runtime_analysis}
\begin{table}[h]
    \centering
    \caption{Parameter counts (thousands) and training time per epoch for each
    depth-recurrence configuration. Runtime is measured for LRU with
    $\mathcal{L}^{\mathrm{final}}$ on an NVIDIA A40 GPU. All configurations use
    depth $L=6$.}
    \label{tab:efficiency}
    \resizebox{\textwidth}{!}{%
    \begin{tabular}{l|cc|cc|cc|cc}
        \toprule
        & \multicolumn{2}{c|}{\texttt{ABCDEF}} 
        & \multicolumn{2}{c|}{\texttt{AAAAAA}} 
        & \multicolumn{2}{c|}{\texttt{ABABAB}} 
        & \multicolumn{2}{c}{\texttt{ABCABC}} \\
        Dataset & Params (K) & Time (s) & Params (K) & Time (s) & Params (K) & Time (s) & Params (K) & Time (s) \\
        \midrule
        Heartbeat  & 154.6 & 8.12  & 29.2 & 5.81  & 54.3 & 6.17  & 79.4 & 7.69  \\
        Worms      & 151.3 & 211.82 & 25.9 & 209.32 & 51.0 & 210.02 & 76.0 & 210.54 \\
        % Ethanol    & 151.1 & 22.95 & 25.6 & 22.28 & 50.7 & 22.50 & 75.8 & 22.51 \\
        Motor      & 154.8 & 36.63 & 29.4 & 35.60 & 54.5 & 35.64 & 79.6 & 35.71 \\
        SCP1       & 151.1 & 12.04 & 25.7 & 11.18 & 50.8 & 11.41 & 75.8 & 11.48 \\
        SCP2       & 151.2 & 15.97 & 25.7 & 14.48 & 50.8 & 14.78 & 75.9 & 14.81 \\
        \midrule
        % Reduction  & $1\times$ & --- & ${\sim}1/6\times$ & ${\sim}1\times$ & ${\sim}1/3\times$ & ${\sim}1\times$ & ${\sim}1/2\times$ & ${\sim}1\times$ \\
        Param. ratio & $1\times$ & --- &
        ${\sim}1/6\times$ & --- &
        ${\sim}1/3\times$ & --- &
        ${\sim}1/2\times$ & --- \\
        \bottomrule
    \end{tabular}
    }
\end{table}

\begin{table}[h]
    \centering
    \caption{Computational cost of different layer configurations.}
    \label{tab:computational_cost}
    \small
    \begin{tabular}{lcccc}
        \toprule
        Config. & Params. & FLOPs & Latency & Peak Memory \\
        \midrule
        \texttt{ABCDEF} (6 independent) & 151-155k & $1.00\times$ & $1.00\times$ & $1.00\times$ \\
        \midrule
        \texttt{ABCABC} (3 unique $\times$ 2) & 76-80k & $\approx1.00\times$ & $\approx1.00\times$ & $\approx1.00\times$ \\
        \texttt{ABABAB} (2 unique $\times$ 3) & 51-54k & $\approx1.00\times$ & $\approx1.00\times$ & $\approx1.00\times$ \\
        \texttt{AAAAAA} (1 unique $\times$ 6) & 26-29k & $\approx1.00\times$ & $\approx1.00\times$ & $\approx1.00\times$ \\
        \midrule
        \texttt{AA} (2 adaptive layers) & 26-29k & 0.35$\times$ & 0.33$\times$ & 0.55$\times$ \\
        \bottomrule
    \end{tabular}
\end{table}

Table~\ref{tab:efficiency} reports parameter counts and measured training time
for the fixed-depth configurations. Depth recurrence reduces the number of
parameters substantially, with \texttt{AAAAAA} using approximately
$5.3-5.9\times$ fewer parameters than \texttt{ABCDEF}. The reduction is
slightly smaller than the ideal $6\times$ factor because the input and output
projections are not shared.

At fixed depth $L=6$, depth recurrence does not reduce the number of block
applications. Consequently, the dominant forward-pass computation remains
approximately unchanged across recurrence periods, as summarized in
Table~\ref{tab:computational_cost}. The measured training times are also
similar, with modest reductions for the depth-recurrent configurations.

Adaptive inference changes this tradeoff by reducing the number of block
applications. Computational cost therefore decreases with the selected
inference depth; for example, applying the \texttt{AAAAAA} model only twice
(\texttt{AA}) requires substantially fewer FLOPs than the fixed six-application
setting. Thus, depth recurrence primarily provides parameter savings at fixed
depth, while adaptive inference can additionally reduce computation.

\subsection{Additional Input Reshaping Results}\label{sec:additional_inputreshaping}

Figure~\ref{fig:reshaping} illustrates the input-reshaping operation for
different input dimensionalities.

\begin{figure}[h]
\centering
\begin{minipage}[c]{0.4\linewidth}
\centering
\includegraphics[width=\linewidth]{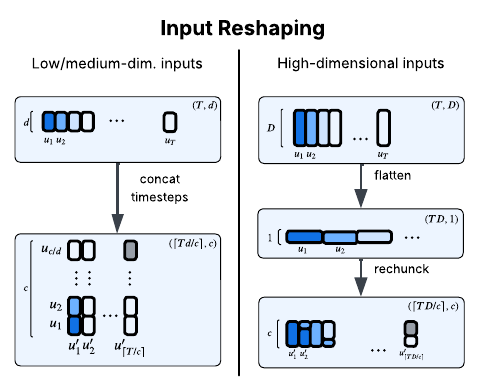}
\end{minipage}
\hfill
\begin{minipage}[c]{0.48\linewidth}
\caption{Input reshaping with concentration factor $c$. For lower- and medium-dimensional
inputs, consecutive timesteps are concatenated; for high-dimensional inputs,
the joint feature-time representation is flattened and rechunked.}
\label{fig:reshaping}
\end{minipage}
\end{figure}

\subsubsection{Complete Numerical Results for Input Reshaping}\label{sec:app_inputreshaping_numerical}

The upper block of Table~\ref{tab:SSM_looped_reshaped} reports the complete results for the datasets evaluated in the main text, both with and without depth recurrence. The lower block reports additional results on Worms, SCP1, and SCP2, where input reshaping by concatenating two consecutive timesteps does not provide consistent or substantial improvements across models. Notably, this behavior cannot be attributed to input dimensionality alone, as the medium-dimensional PPG-DaLiA dataset in the upper block benefits substantially from reshaping.

Across the upper-block datasets, 51 of 72 (70.8\%) depth-recurrent configurations without reshaping exceed the mean performance of the corresponding original-input baseline (see Table~\ref{tab:SSM_looping}). When depth recurrence is combined with input reshaping (Table~\ref{tab:SSM_looped_reshaped}), this increases to 57 of 72 (79.2\%) configurations. Moreover, 69 of 72 (95.8\%) combined configurations are either numerically better than the original-input baseline or remain within one standard deviation of its mean.

In contrast, for Worms, SCP1, and SCP2 in the lower block, only 24 of 72 (33.3\%) combined configurations have a numerically better mean than the original-input baseline. Nevertheless, 62 of 72 (86.1\%) are either numerically better or remain within one standard deviation of the baseline mean. Overall, these results suggest that the benefits of input reshaping are dataset-dependent rather than determined solely by input dimensionality.

\begin{table}[h]
    \centering
        \caption{Performance with input reshaping and depth recurrence. The first
        \texttt{ABCDEF} column gives the original-input baseline and the second the
        reshaped-input baseline. Depth-recurrent configurations use the same reshaped
        input. Results are mean $\pm$ std over 5 seeds. Blue indicates an improved mean
        of the reshaped baseline over the original baseline; bold indicates an
        improved mean over the reshaped baseline. The lower block contains the medium-dimensional datasets, for which input reshaping did not yield substantial improvements over the original-input baseline.}
    \label{tab:SSM_looped_reshaped}
    \resizebox{\textwidth}{!}{%
    \begin{tabular}{c|c|c|c|cccccc}
        \toprule
        % Dataset & Model & Reshaped Baseline & & & & & & \\
             & & Baseline &  Reshaped-Input Baseline & \multicolumn{6}{c}{Depth recurrence with reshaped input}\\
            Dataset & Model & \texttt{ABCDEF} & \texttt{ABCDEF} & \multicolumn{2}{c}{\texttt{AAAAAA}} & \multicolumn{2}{c}{\texttt{ABABAB}} & \multicolumn{2}{c}{\texttt{ABCABC}}\\
            &&  $\mathcal{L}^{\text{final}}$ & $\mathcal{L}^{\text{final}}$& $\mathcal{L}^{\text{final}}$ & $\mathcal{L}^{\text{block}}$ & $\mathcal{L}^{\text{final}}$ & $\mathcal{L}^{\text{block}}$ & $\mathcal{L}^{\text{final}}$ & $\mathcal{L}^{\text{block}}$ \\
        \midrule
        % \multirow{4}{*}{Ethanol ($\uparrow$)}
        %     & S5   & $19.24 \pm 6.95$   & $\textcolor{blue}{\mathbf{24.56 \pm 4.64}}$ & $\underline{24.56 \pm 4.64}$ & $\underline{24.56 \pm 4.64}$ & $\underline{24.56 \pm 4.64}$ & $\underline{24.56 \pm 4.64}$ & $\underline{24.56 \pm 4.64}$ & $\underline{24.56 \pm 4.64}$\\
        %     & LRU  & $27.09 \pm 3.63$   & $\textcolor{blue}{\mathbf{27.59 \pm 6.12}}$ & $\underline{27.59 \pm 6.12}$ & $\underline{27.59 \pm 6.12}$ & $\underline{27.59 \pm 6.12}$  & $25.82 \pm 6.33$ & $25.82 \pm 6.33$ & $\underline{27.59 \pm 6.12}$\\
        %     & LinOSS & $23.73 \pm 7.82$  & $\textcolor{blue}{\mathbf{24.56 \pm 4.64}}$ & $\underline{24.56 \pm 4.64}$ & $\underline{24.56 \pm 4.64}$ & $\underline{24.56 \pm 4.64}$ & $\underline{24.56 \pm 4.64}$  & $\underline{24.56 \pm 4.64}$ & $\underline{24.56 \pm 4.64}$\\
        %     & LrcSSM & $36.05 \pm 2.00$ & $\textcolor{blue}{\mathbf{39.21 \pm 3.00}}$ & $37.11 \pm 6.06$ & $33.68 \pm 7.54$ & $37.63 \pm 2.39$ & $36.84 \pm 3.09$ & $37.63 \pm 3.56$ &  $36.84 \pm 2.63$\\
        % \midrule
        \multirow{4}{*}{Motor ($\uparrow$)}
            & S5  & $54.39 \pm 5.77$    & $\textcolor{blue}{\mathbf{61.75 \pm 5.25}}$ & $60.00 \pm 6.02$ & $58.60 \pm 4.24$ & $59.65 \pm 5.77$ & $60.00 \pm 5.37$ & $60.70 \pm 8.35$ & $60.00 \pm 7.31$\\
            & LRU   & $56.14 \pm 6.37$   & $\textcolor{blue}{\mathbf{58.60 \pm 6.04}}$ & $53.33 \pm 3.06$ & $56.84 \pm 5.16$ & $56.14 \pm 7.28$ & $54.74 \pm 6.79$ & $57.54 \pm 6.79$ & $55.79 \pm 6.51$\\
            & LinOSS &  $49.12 \pm 2.77$  & $\textcolor{blue}{\mathbf{52.28 \pm 6.96}}$ & $\mathbf{55.44 \pm 5.72}$ & $\mathbf{55.44 \pm 3.44}$ & $\mathbf{52.63 \pm 3.84}$ & $\mathbf{54.39 \pm 7.92}$  &  $\mathbf{55.09 \pm 3.61}$& $\mathbf{54.74 \pm 5.25}$\\
            & LrcSSM &  $57.50 \pm 3.43$  & $57.14 \pm 8.66$ & $56.07 \pm 3.24$ & $55.71 \pm 5.11$  & $\mathbf{60.71 \pm 6.06}$ & $52.86 \pm 3.70$ & $\mathbf{60.27 \pm 6.08}$ & $55.71 \pm 8.22$\\
        \midrule
        \multirow{4}{*}{Heartbeat ($\uparrow$)}
            & S5   & $74.19 \pm 4.67$    & $74.19 \pm 3.38$ & $\mathbf{76.13 \pm 2.37}$ & $\mathbf{75.80 \pm 2.28}$ & $\mathbf{76.45 \pm 2.62}$  & $\mathbf{75.81 \pm 4.33}$ & $\mathbf{75.81 \pm 2.28}$ & $\mathbf{75.81 \pm 2.70}$\\
            & LRU &  $71.29 \pm 1.21$    & $\textcolor{blue}{\mathbf{75.48 \pm 2.37}}$ & $72.90 \pm 4.61$ & $73.87 \pm 3.59$ & $\mathbf{75.81 \pm 2.89}$ & $\mathbf{76.13 \pm 3.44}$ & $75.16 \pm 3.62$ & $75.16 \pm 4.16$\\
            & LinOSS  & $70.16 \pm 1.80$  & $\textcolor{blue}{\mathbf{70.65 \pm 6.40}}$ & $\mathbf{71.29 \pm 3.44}$ & $\mathbf{71.61 \pm 4.16}$ & $70.32 \pm 3.76$ &  $\mathbf{70.97 \pm 4.08}$&  $\mathbf{72.26 \pm 4.49}$ & $\mathbf{73.23 \pm 3.76}$\\
            & LrcSSM & $74.67 \pm 5.2$  & $\textcolor{blue}{\mathbf{77.67 \pm 6.08}}$ & $75.00 \pm 4.56$ & $76.00 \pm 6.08$ & $75.67 \pm 5.08$ & $75.33 \pm 5.45$ & $76.00 \pm 7.51$ & $75.33 \pm 4.77$\\
        \midrule
        \midrule
        \multirow{4}{*}{PPG-DaLiA ($\downarrow$)}
        &S5
        & $6.83 \pm 0.27$
        & $\textcolor{blue}{\mathbf{6.77 \pm 0.24}}$
        & $6.87 \pm 0.53$
        & $7.47 \pm 0.15$
        & $\mathbf{6.48 \pm 0.66}$
        & $7.52 \pm 1.13$
        & $\mathbf{6.57 \pm 0.37}$
        & $\mathbf{6.38 \pm 0.61}$ \\
        &LRU
        & $9.33 \pm 0.44$
        & $\textcolor{blue}{\mathbf{7.07 \pm 1.51}}$
        & $7.72 \pm 0.73$
        & $8.01 \pm 0.33$
        & $8.30 \pm 0.91$
        & $8.08 \pm 0.24$
        & $8.09 \pm 0.62$
        & $7.84 \pm 1.04$ \\
        &LinOSS
        & $8.33 \pm 0.84$
        & $8.96 \pm 1.08$
        & $\mathbf{7.80 \pm 0.15}$
        & $\mathbf{8.31 \pm 0.41}$
        & $\mathbf{8.15 \pm 0.59}$
        & $\mathbf{8.09 \pm 0.33}$
        & $8.88 \pm 1.26$
        & $\mathbf{8.30 \pm 0.73}$ \\
        &LrcSSM
        & $9.20 \pm 1.00$
        & $\textcolor{blue}{\mathbf{8.60 \pm 0.59}}$
        & $8.96 \pm 0.16$
        & $9.46 \pm 1.28$
        & $9.21 \pm 0.67$
        & $9.45 \pm 1.21$
        & $8.88 \pm 1.08$
        & $\mathbf{7.79 \pm 0.36}$ \\
        \midrule
        \addlinespace[0.5em]
        \midrule
        \multirow{4}{*}{Worms ($\uparrow$)}
        & LRU
        & $88.33 \pm 4.08$
        & $\textcolor{blue}{\mathbf{88.89 \pm 6.33}}$
        & $87.78 \pm 6.48$ & $\mathbf{90.00 \pm 7.58}$
        & $\mathbf{91.11 \pm 5.09}$ & $\mathbf{90.56 \pm 7.97}$
        & $\mathbf{91.11 \pm 5.93}$ & $\mathbf{88.89 \pm 5.83}$\\
        & S5
        & $86.67 \pm 5.10$
        & $83.89 \pm 4.44$
        & $\mathbf{85.00 \pm 4.51}$ & $83.33 \pm 5.83$
        & $82.78 \pm 3.24$ & $\mathbf{85.56 \pm 2.72}$
        & $\mathbf{84.44 \pm 5.15}$ & $82.22 \pm 4.51$\\
        & LinOSS
        & $95.83 \pm 3.11$
        & $70.00 \pm 31.93$
        & $\mathbf{94.44 \pm 3.04}$ & $\mathbf{91.67 \pm 4.65}$
        & $\mathbf{88.33 \pm 16.52}$ & $\mathbf{93.33 \pm 5.72}$
        & $\mathbf{77.78 \pm 31.38}$ & $\mathbf{86.67 \pm 13.77}$\\
        & LrcSSM
        & $83.33 \pm 5.56$
        & $80.00 \pm 7.19$
        & $\mathbf{82.78 \pm 4.12}$ & $\mathbf{82.78 \pm 6.02}$
        & $\mathbf{82.78 \pm 2.32}$ & $\mathbf{83.89 \pm 4.56}$
        & $\mathbf{82.22 \pm 1.52}$ & $\mathbf{84.44 \pm 5.76}$\\

        \midrule
        \multirow{4}{*}{SCP1 ($\uparrow$)}
        & LRU
        & $85.65 \pm 2.92$
        & $\textcolor{blue}{\mathbf{86.59 \pm 4.74}}$
        & $86.35 \pm 3.12$ & $85.18 \pm 2.42$
        & $\mathbf{87.06 \pm 3.87}$ & $85.88 \pm 3.79$
        & $\mathbf{87.29 \pm 4.30}$ & $84.94 \pm 3.75$\\
        & S5
        & $90.12 \pm 2.05$
        & $88.94 \pm 1.19$
        & $87.76 \pm 2.53$ & $86.35 \pm 2.74$
        & $88.47 \pm 2.92$ & $86.35 \pm 4.25$
        & $87.06 \pm 3.72$ & $86.35 \pm 3.98$\\
        & LinOSS
        & $85.88 \pm 4.27$
        & $85.18 \pm 4.12$
        & $\mathbf{87.06 \pm 1.29}$ & $\mathbf{85.41 \pm 3.98}$
        & $\mathbf{86.12 \pm 5.92}$ & $84.94 \pm 1.37$
        & $\mathbf{86.59 \pm 2.64}$ & $85.18 \pm 3.98$\\
        & LrcSSM
        & $84.52 \pm 4.29$
        & $\textcolor{blue}{\mathbf{84.76 \pm 1.00}}$
        & $\mathbf{85.71 \pm 4.37}$ & $83.33 \pm 4.61$
        & $84.52 \pm 2.23$ & $\mathbf{85.00 \pm 1.99}$
        & $\mathbf{85.00 \pm 1.06}$ & $\mathbf{85.48 \pm 2.84}$\\

        \midrule
        \multirow{4}{*}{SCP2 ($\uparrow$)}
        & LRU
        & $55.09 \pm 5.04$
        & $\textcolor{blue}{\mathbf{55.09 \pm 3.61}}$
        & $\mathbf{56.14 \pm 1.92}$ & $53.33 \pm 7.33$
        & $\mathbf{56.84 \pm 2.38}$ & $52.98 \pm 6.12$
        & $\mathbf{56.84 \pm 1.79}$ & $54.74 \pm 10.01$\\
        & S5
        & $54.74 \pm 5.25$
        & $\textcolor{blue}{\mathbf{55.09 \pm 3.94}}$
        & $53.68 \pm 3.06$ & $53.33 \pm 2.38$
        & $55.09 \pm 3.06$ & $52.63 \pm 2.48$
        & $\mathbf{55.44 \pm 2.85}$ & $54.74 \pm 5.81$\\
        & LinOSS
        & $54.39 \pm 6.37$
        & $\textcolor{blue}{\mathbf{56.49 \pm 5.59}}$
        & $51.93 \pm 1.79$ & $51.23 \pm 5.70$
        & $54.04 \pm 6.11$ & $54.39 \pm 5.32$
        & $\mathbf{57.89 \pm 5.55}$ & $52.28 \pm 2.05$\\
        & LrcSSM
        & $57.14 \pm 5.79$
        & $51.07 \pm 1.60$
        & $50.71 \pm 8.24$ & $\mathbf{55.00 \pm 8.51}$
        & $\mathbf{54.29 \pm 9.83}$ & $\mathbf{55.36 \pm 9.11}$
        & $\mathbf{53.21 \pm 7.41}$ & $\mathbf{52.50 \pm 5.30}$\\
        \bottomrule
    \end{tabular}
    }
    \vspace*{-3ex}
\end{table}

\subsubsection{Effect of the Concentration Factor}
We additionally vary the concentration factor $c$ using LrcSSM, with results
reported in Table~\ref{tab:inputdim_comparison}. For the medium-dimensional
datasets, concatenating two complete timesteps ($c=2d$) performs better than
the fixed choices $c\in\{8,16,32\}$ overall. For the high-dimensional
Heartbeat and Motor datasets, $c=8$ gives the strongest results among the
evaluated values. We also compared zero-padding with discarding trailing
values when $TD$ is not divisible by $c$ and observed comparable performance;
we therefore use zero-padding to retain all input values.

\begin{table}[t]
    \centering
    \caption{Classification accuracy (\%) for different concentration factors
    $c$ using LrcSSM, reported as mean $\pm$ std over 5 seeds.}
    \label{tab:inputdim_comparison}
    \begin{tabular}{l|cccc}
        \toprule
        Dataset
        & $c = 2d$
        & $c = 8$
        & $c = 16$
        & $c = 32$ \\
        \midrule
        % Ethanol   & $36.89 \pm 2.92$ (4)  & $39.21 \pm 3.00$ & $37.63 \pm 2.56$ & $36.84 \pm 1.61$ \\
        Worms     & $80.00 \pm 7.19$ (12) & $78.33 \pm 6.92$ & $78.89 \pm 8.24$ & $76.67 \pm 2.48$ \\
        SCP1      & $84.76 \pm 1.00$ (12) & $79.29 \pm 4.88$ & $80.48 \pm 2.47$ & $76.90 \pm 5.62$ \\
        SCP2      & $51.07 \pm 1.60$ (14) & $48.57 \pm 4.96$ & $47.14 \pm 7.21$ & $50.00 \pm 8.09$ \\
        \midrule
        Heartbeat & Too large            & $77.67 \pm 6.08$ & $75.00 \pm 4.56$ & $71.33 \pm 7.76$ \\
        Motor     & Too large            & $57.14 \pm 8.66$ & $54.29 \pm 4.82$ & $54.64 \pm 5.45$ \\
        \bottomrule
    \end{tabular}
\end{table}

\subsection{Parameter-Matched Ablations}\label{sec:parameter_matched}
\paragraph{Depth-recursion Under Parameter-matched Conditions}
The experiments in Section~\ref{sec:looped_ssm} compare depth-recurrent models
against an independently parameterized 6-layer baseline with more parameters.
To isolate the effect of repeated computation from parameter count, we
additionally compare each LRU configuration with an independent model
containing the same number of unique blocks:
\texttt{A} vs.\ \texttt{AAAAAA},
\texttt{AB} vs.\ \texttt{ABABAB}, and
\texttt{ABC} vs.\ \texttt{ABCABC}.

Figure~\ref{fig:lru_param_matched_bars_highlighted} shows that depth recurrence
can improve over these parameter-matched baselines. The clearest gains occur
for the most constrained configuration: applying a single block six times
(\texttt{AAAAAA}, $\mathcal{L}^{\mathrm{final}}$) improves the mean over the
single-block model by $5.6\%$ on Motor and $4.5\%$ on
Worms. Improvements also occur with two and three unique blocks, although
their magnitude depends on the dataset and supervision strategy. Overall,
these results show that the performance of depth-recurrent models cannot be
explained by their lower parameter count alone.

\begin{figure}
    \centering
    \includegraphics[width=0.8\linewidth]{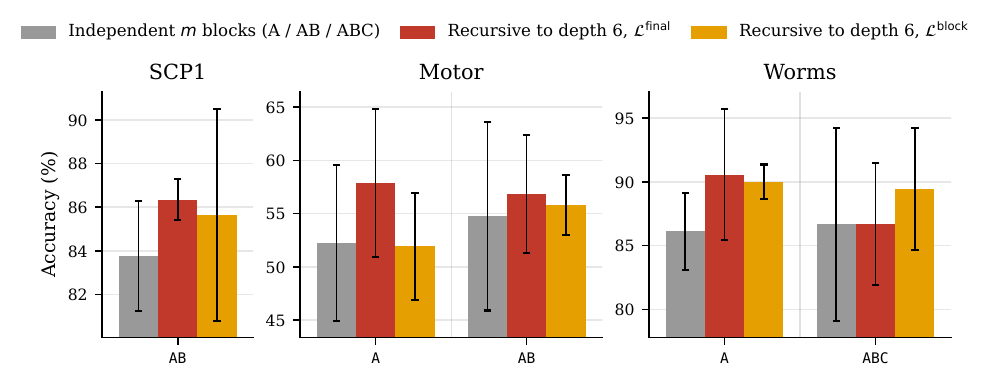}
    \caption{LRU under parameter-matched conditions. Depth-recurrent
        configurations are compared with independent models containing the same
        number of unique blocks. Mean $\pm$ std over 5 seeds.}
    \label{fig:lru_param_matched_bars_highlighted}
\end{figure}

\paragraph{Low-Capacity Architectures}

We further evaluate whether reducing model depth alone can account for the
observed behavior using Mamba~\citep{gu2023mamba} and MambaSL~\citep{jung2026mambasl},
a Mamba variant designed for stronger single-layer performance.
Table~\ref{tab:mamba_layer_config} compares the independently parameterized
6-layer model (\texttt{ABCDEF}), a single-block model (\texttt{A}), and the
same block applied six times (\texttt{AAAAAA}). MambaSL improves over Mamba
in the single-block setting on both datasets, consistent with its
single-layer design. However, reducing the architecture to a single block
does not improve over the 6-layer baseline. Reusing that block across depth,
in contrast, recovers or improves mean performance without increasing the number
of unique SSM-block parameters. This provides additional evidence that the
benefit of depth recurrence is not explained by parameter reduction alone.

\begin{table}[t]
    \centering
    \caption{Classification accuracy (\%) for Mamba~\citep{gu2023mamba} and
    MambaSL~\citep{jung2026mambasl} with six independent blocks
    (\texttt{ABCDEF}), one block (\texttt{A}), or one block applied six times
    (\texttt{AAAAAA}). Results are mean $\pm$ std over 5 seeds.}
    \label{tab:mamba_layer_config}
    % \resizebox{\textwidth}{!}{%
    \small
    \begin{tabular}{l|l|cccc}
        \toprule
         \multirow{3}{*}{Dataset} & \multirow{3}{*}{Model}
        & \multicolumn{2}{c}{Baselines}
         & \multicolumn{2}{c}{Depth-recursion}\\
         & 
        & \texttt{ABCDEF}
        & Single-layer \texttt{A}
        & \texttt{AAAAAA}
        & \texttt{AAAAAA} \\
         & 
        & 
        & 
        &  $\mathcal{L}^{\text{final}}$
        &  $\mathcal{L}^{\text{block}}$ \\
        \midrule
        \multirow{2}{*}{Heartbeat ($\uparrow$)} & Mamba
        & $72.58 \pm 3.38$
        & $69.67 \pm 5.34$
        & $\mathbf{73.23 \pm 1.94}$
        & $\mathbf{73.23 \pm 3.75}$ \\
         & MambaSL
        & $70.97 \pm 5.30$
        & $71.61 \pm 5.36$
        & $\mathbf{72.25 \pm 3.73}$
        & $\mathbf{72.58 \pm 3.38}$ \\
        \midrule
        \multirow{2}{*}{SCP1 ($\uparrow$)} & Mamba
        & $82.12 \pm 2.72$
        & $78.35 \pm 3.76$
        & $\mathbf{83.05 \pm 4.25}$
        & $\mathbf{84.71 \pm 1.29}$ \\
         & MambaSL
        & $84.71 \pm 3.65$
        & $83.29 \pm 4.17$
        & $\mathbf{85.41 \pm 1.76}$
        & $\mathbf{84.71 \pm 3.07}$ \\
        \bottomrule
    \end{tabular}
    % }
\end{table}

\subsection{Additional Adaptive-Inference Results}
\label{sec:additional_adaptive}

We additionally evaluate adaptive inference for the LRU
\texttt{ABABAB} configuration, complementing the \texttt{AAAAAA} results in
Section~\ref{sec:adaptive}. The model is trained with three repetitions of
\texttt{AB} and, at inference time, is allowed to execute up to six
repetitions (12 block applications). Table~\ref{tab:ababab_early_stopping}
compares stopping after any block application with stopping only after a
complete \texttt{AB} repetition.

Adaptive inference selects different depths across datasets. When stopping is
restricted to complete \texttt{AB} repetitions, the average inference depth
ranges from approximately 4.4 to 11.1 block applications, compared with the
fixed training depth of six. Performance remains close to or even slightly exceeds the
fixed-depth baseline on most datasets, although the behavior is
dataset-dependent.

We further consider two ablations that restrict adaptive inference to depths
of at least six or at most six block applications. Neither restriction
consistently improves performance across datasets, indicating that allowing
the stopping criterion to select depths both below and above the training
depth provides the most flexible inference setting.

\begin{table*}[h]
    \centering
    \caption{Adaptive inference for LRU trained with the \texttt{ABABAB}
    configuration under $\mathcal{L}^{\mathrm{block}}$. The model may execute
    up to 12 block applications at inference time. Results are classification
    accuracy (\%), reported as mean $\pm$ std over 5 seeds. Bold indicates an
    improved mean over fixed-depth \texttt{ABABAB}.}
    \label{tab:ababab_early_stopping}
    \resizebox{\textwidth}{!}{%
    \begin{tabular}{lc|ccc|cc}
        \toprule
        Dataset
        & \texttt{ABABAB}
        & Stop at any block
        & Stop after full \texttt{AB}
        & Avg. depth
        & Min. depth 6
        & Max. depth 6 \\
        & $\mathcal{L}^{\text{block}}$
        & (1-12)
        & (2, 4, \ldots, 12)
        &
        & (6-12)
        & (1-6) \\
        \midrule
        Heartbeat & $70.32 \pm 4.16$ & $\mathbf{72.26 \pm 4.13}$ & $\mathbf{72.90 \pm 4.26}$ & $5.25 \pm 0.26$ & $65.16 \pm 4.16$ & $\mathbf{71.61 \pm 3.32}$ \\
        SCP1      & $85.65 \pm 4.85$ & $79.29 \pm 6.68$ & $85.41 \pm 4.05$ & $5.04 \pm 0.49$ & $81.88 \pm 3.03$ & $\mathbf{86.35 \pm 4.97}$ \\
        SCP2      & $55.09 \pm 9.32$ & $\mathbf{57.19 \pm 4.79}$ & $\mathbf{56.14 \pm 6.18}$ & $11.08 \pm 1.23$ & $\mathbf{56.49 \pm 6.88}$ & $55.09 \pm 6.04$ \\
        % Ethanol   & $26.84 \pm 3.53$ & $\mathbf{28.10 \pm 4.63}$ & $\mathbf{28.10 \pm 4.63}$ & $4.00 \pm 0.00$ & $\mathbf{28.10 \pm 4.63}$ & $\mathbf{28.10 \pm 4.63}$ \\
        Motor     & $55.79 \pm 2.81$ & $\mathbf{57.89 \pm 4.44}$ & $\mathbf{57.54 \pm 3.40}$ & $10.50 \pm 0.36$ & $\mathbf{57.54 \pm 3.02}$ & $54.39 \pm 2.72$ \\
        Worms     & $90.00 \pm 3.77$ & $87.22 \pm 5.72$ & $89.44 \pm 6.89$ & $4.43 \pm 0.18$ & $86.67 \pm 5.39$ & $90.00 \pm 5.15$ \\
        \bottomrule
    \end{tabular}
    }
\end{table*}

\subsection{Damped Iterative Updates}
\label{sec:damped_iterative_updates}

Our depth-recurrent models use the direct iterative update
\begin{equation}
    h^{(r+1)} = f_{\theta}\!\left(h^{(r)}\right),
\end{equation}
where $r$ denotes the depth iteration and $f_{\theta}$ is the shared SSM
block. We investigate whether a more conservative update improves repeated
application of the same block.

Specifically, we consider a damped update based on the
Krasnosel'skii--Mann iteration~\citep{krasnosel1955two, mann1953mean}:
\begin{equation}
    h^{(r+1)}
    =
    (1-\alpha)h^{(r)}
    +
    \alpha f_{\theta}\!\left(h^{(r)}\right),
    \qquad \alpha\in[0,1].
\end{equation}
Here, $\alpha$ controls how strongly each iteration incorporates the newly
computed representation, with $\alpha=1$ recovering the direct update used
in our main experiments. We evaluate both fixed values of $\alpha$ and a
learned $\alpha$.

Table~\ref{tab:alpha_comparison} reports the results for the LRU
\texttt{AAAAAA} configuration. Damping improves the mean accuracy on Motor,
with the strongest fixed setting ($\alpha=0.75$) improving from $51.93\%$ to
$56.84\%$. However, it does not provide consistent improvements across the
other datasets, and learning $\alpha$ does not consistently improve over the fixed update. We therefore use the simpler $\alpha=1$ update throughout the
main experiments.

\begin{table*}[h]
    \centering
    \caption{Classification accuracy (\%) for LRU with the
    \texttt{AAAAAA} configuration under damped depth updates. The direct
    update used in the main experiments corresponds to $\alpha=1$.
    Results are mean $\pm$ std over 5 seeds.}
    \label{tab:alpha_comparison}
    \begin{tabular}{l|ccccc}
        \toprule
        Dataset & $\alpha=1$ & $\alpha=0.25$ & $\alpha=0.75$
        & $\alpha=0.9$ & Learned $\alpha$ \\
        \midrule
        Heartbeat & $72.90 \pm 5.34$ & $71.61 \pm 4.63$ & $71.61 \pm 2.19$ & $71.94 \pm 3.90$ & $72.25 \pm 3.73$ \\
        SCP1      & $84.94 \pm 2.72$ & $84.00 \pm 3.29$ & $83.53 \pm 2.98$ & $84.94 \pm 3.28$ & $84.47 \pm 3.28$ \\
        SCP2      & $54.04 \pm 7.56$ & $53.33 \pm 8.60$ & $53.33 \pm 5.83$ & $54.39 \pm 7.19$ & $51.93 \pm 5.91$ \\
        Motor     & $51.93 \pm 5.04$ & $55.79 \pm 5.48$ & $56.84 \pm 3.78$ & $53.33 \pm 4.24$ & $56.14 \pm 2.48$ \\
        Worms     & $90.00 \pm 1.36$ & $87.78 \pm 3.77$ & $87.22 \pm 4.51$ & $88.33 \pm 3.24$ & $87.22 \pm 9.06$ \\
        \bottomrule
    \end{tabular}
\end{table*}
\end{document}

%% file: ICLR2027/math_commands.tex
\usepackage{amsmath,amsfonts,bm}

\def\eqref#1{equation~\ref{#1}}
\def\Eqref#1{Equation~\ref{#1}}
\def\1{\bm{1}}

\DeclareMathAlphabet{\mathsfit}{\encodingdefault}{\sfdefault}{m}{sl}
\SetMathAlphabet{\mathsfit}{bold}{\encodingdefault}{\sfdefault}{bx}{n}